\begingroup\expandafter\expandafter\expandafter\endgroup
\expandafter\ifx\csname pdfsuppresswarningpagegroup\endcsname\relax
\else
\fi
\documentclass[a4paper,10pt]{article}
\usepackage[utf8]{inputenc}
\usepackage{tikz}
\usepackage{amsfonts,bm}
\usepackage{amsmath}
\usepackage{mathtools}
\usepackage{latexsym}
\usepackage{amssymb}
\PassOptionsToPackage{hyphens}{url}\usepackage[hidelinks]{hyperref}
\usepackage{graphicx}
\usepackage{tcolorbox}
\usepackage[numbers,sort&compress]{natbib}

\DeclareMathOperator*{\argmax}{arg\,max}
\usepackage{float,xspace}
\usepackage{cprotect}
\usepackage[capitalise]{cleveref}
\usepackage{tcolorbox}
\usepackage[ruled,vlined,linesnumbered,algo2e]{algorithm2e}
\usepackage{float,subfig}
\usepackage{forloop}
\newcommand{\defcal}[1]{\expandafter\newcommand\csname c#1\endcsname{{\mathcal{#1}}}}
\newcommand{\defbb}[1]{\expandafter\newcommand\csname b#1\endcsname{{\mathbb{#1}}}}
\newcounter{calBbCounter}
\forLoop{1}{26}{calBbCounter}{
	\edef\letter{\Alph{calBbCounter}}
	\expandafter\defcal\letter
	\expandafter\defbb\letter
}
\usepackage[T1]{fontenc}    
\usepackage{booktabs}       
\usepackage{amsfonts}       
\usepackage{nicefrac}       
\usepackage{microtype}      
\usepackage[mmddyyyy]{datetime}
\usepackage{amsthm}
\newtheorem{theorem}{Theorem}

\newtheorem{observation}[theorem]{Observation}
\newtheorem{definition}[theorem]{Definition}
\newtheorem{lemma}[theorem]{Lemma}
\newtheorem{corollary}[theorem]{Corollary}

\newcommand{\Opt}{\text{\upshape{\texttt{OPT}}}\xspace}
\newcommand{\Optset}{OPT}

\newcommand{\AlgThreshold}{{\textsc{\textsc{Threshold-Streaming}}}\xspace}
\newcommand{\AlgStream}{{\textsc{\textsc{Distorted-Streaming}}}\xspace}
\newcommand{\AlgDistributed}{{\textsc{\textsc{Distorted-Distributed-Greedy}}}\xspace}
\newcommand{\AlgDistorted}{{\textsc{\textsc{Distorted-Greedy}}}\xspace}

\DeclareMathOperator{\ground}{\cN}

\crefname{algocf}{Algorithm}{Algorithms}
\Crefname{algocf}{Algorithm}{Algorithms}

\renewcommand{\emptyset}{\varnothing}
\usepackage{numprint}
\usepackage{xifthen}
\npthousandsep{,}
\usepackage[top=1.in, bottom=1in, left=1in, right=1in]{geometry}
\usepackage{paralist, tabularx}
\makeatletter
\newcommand{\newreptheorem}[2]{%
	\newenvironment{rep#1}[1]{%
		\expandafter\renewcommand\csname the#2\endcsname{\ref*{##1}}%
		\expandafter\renewcommand\csname theH#2\endcsname{repeat.##1}%
		\begin{#1}}%
		{\end{#1}%
		\addtocounter{#2}{-1}}}
\makeatother
\newreptheorem{theorem}{theorem}
\newreptheorem{observation}{theorem}
\newreptheorem{lemma}{theorem}
\newreptheorem{proposition}{theorem}
\newreptheorem{corollary}{theorem}
\newreptheorem{reduction}{theorem}
\newcommand{\eps}{\varepsilon}

\renewcommand{\emptyset}{\varnothing}
\renewcommand{\epsilon}{\varepsilon}
\newcommand{\lov}[1]{{\hat{#1}}}
\AtBeginDocument{\renewcommand\cite{\citep}}
\newcommand{\printtime}{01/01/2019}

\title{Regularized Submodular Maximization at Scale}
\author
{
	Ehsan Kazemi\thanks{Yale Institute for Network Science, Yale University. Email: \texttt{ehsan.kazemi@yale.edu}.}
	\and
	Shervin Minaee\thanks{Expedia Group, Seattle, WA. Email: \texttt{shervin.minaee@nyu.edu}.}
	\and
	Moran Feldman\thanks{Department of Computer Science, University of Haifa, Israel. Email: \texttt{moranfe@cs.haifa.ac.il}.}
	\and
	Amin Karbasi\thanks{Yale Institute for Network Science, Yale University. Email: \texttt{amin.karbasi@yale.edu}.}
}
\date{}

\begin{document}
	\pagenumbering{arabic}
	
	\maketitle
	\begin{abstract}	
		In this paper, we propose scalable methods   for maximizing a \textit{regularized} submodular function $f(\cdot) = g(\cdot)- \ell(\cdot)$ expressed as the difference between a monotone submodular function $g$ and a modular function $\ell$. Indeed,  submodularity is inherently related to the notions of diversity, coverage, and representativeness. In particular, finding the mode (i.e., the most likely configuration) of many popular probabilistic models of diversity, such as determinantal point processes, submodular probabilistic models, and strongly log-concave distributions, involves maximization of (regularized) submodular functions. 
	Since a regularized function $f$ can potentially take on negative values, the classic theory of submodular maximization, which heavily relies on the assumption that the submodular function is non-negative, may not be applicable. To circumvent this challenge, we develop \AlgStream, the first one-pass streaming algorithm for maximizing a regularized submodular function subject to a $k$-cardinality constraint. It returns a solution $S$ with the guarantee that $f(S)\geq (\phi^{-2} - \eps) \cdot g(OPT) - \ell(OPT)$, where $\phi$ is the golden ratio (and thus, $\phi^{-2} \approx 0.382$). Furthermore, we develop \AlgDistributed, the first   distributed algorithm that returns a solution $S$ with the guarantee that  $\bE [f(S)] \geq (1 - \eps)  \left[(1 - e^{-1}) \cdot g(OPT) - \ell(OPT) \right]$ in $O(1/\epsilon)$ rounds of MapReduce computation. We should highlight that our result, even for the unregularized case  where the modular term $\ell$ is zero, improves the memory and communication complexity of the existing work by a factor of $O(1/\epsilon)$ as it manages to avoid the need (of this existing work) to keep multiple copies of the entire dataset. Moreover, it does so while (arguably) providing a simpler distributed algorithm and a unifying analysis.  We also empirically study the performance of our scalable methods on a set of real-life applications, including vertex cover of social networks, mode of strongly log-concave distributions, data summarization (such as video summarization,  location summarization, and text summarization), and product recommendation. 
\end{abstract}
	\section{Introduction} \label{sec:intro}
Finding a diverse  set of items, also known as data summarization, is one of the 
central tasks in machine learning.  It usually involves either maximizing a utility 
function that promotes coverage and representativeness \cite{mirzasoleiman2013distributed,wei2015submodularity} (we call this an optimization perspective) 
	or sampling from discrete probabilistic models that promote negative 
	correlations and show repulsive behaviors \cite{rebeschini2015fast,gotovos2015sampling} (we call this a sampling 	perspective). Celebrated examples of probabilistic models that encourage negative dependency include determinantal point processes \cite{kulesza2012determinantal}, 
	strongly Rayleigh measures \cite{borcea2009negative}  strongly log-concave distributions \cite{gurvits2009polynomial}, 
	and probabilistic submodular models \cite{djolonga2014map,iyer2015submodular}. In fact, the two above views are tightly 
related in the sense that oftentimes the mode (i.e., most likely configuration) 
of a diversity promoting 
distribution is  a simple variant of a (regularized) submodular function. For 
instance, determinantal point processes are log-submodular. Or, as we show 
later, a strongly log-concave distribution  is indeed  a regularized log-submodular 
plus a log quadratic term.  The 
	aim of this paper is to show how such an optimization task  can be 
	done at scale. 

From the optimization perspective, in order to effectively select a diverse subset of items, we need to define a measure that captures the amount of representativeness that lies within a selected subset. Oftentimes, such a measure  naturally satisfies  the intuitive diminishing returns condition which can be formally captured by \emph{submodularity}. Given a finite ground set $\cN$ of size $n$, consider a set function $g\colon 2^{\cN} \to \bR$ assigning a utility $g(A)$ to every subset $A \subseteq \cN$. We say that $g$ is submodular if for any pair of subsets $A \subseteq B \subseteq  \ground$ and an element $u \not \in B$, we have
$ g(A \cup \{u\}) - g(A) \geq g(B \cup \{u\}) - g(B), $
which intuitively means that the increase in ``representativeness'' following the 
addition of an element $u$ is smaller when $u$ is added to a larger set. 
Additionally, a set function is said to be monotone if $g(A) \leq g(B)$ for all $A 
\subseteq B \subseteq \cN$; that is, adding more data can only increase the 
representativeness of any subset. Many of the previous work in data 
summarization and diverse subset selection that take an optimization  
perspective, simply aim to maximize a monotone submodular function 
\cite{mirzasoleiman2013distributed}. Monotonicity has the 
	advantage of promoting coverage, but it also enhances the danger of over-fitting to the 
	data as adding more elements can never decrease the utility. To address this 
	issue, as it is often  done in machine learning, we need to add a simple 
	penalty  or  a regularizer term.  Formally, we cast this optimization problem as 
	an instance of \textbf{regularized submodular} maximization, in which we are 
	asked to find a set $S$ of size at most $k$ that  maximizes
\begin{align} \label{eq:problem}
\Optset = \argmax_{S \subseteq \cN, |S| \leq k} [g(S) - \ell(S)] \enspace,
\end{align}
where $g$ is a non-negative monotone submodular
function and $\ell$ is a non-negative modular function.\footnote{A set function 
$\ell \colon 2^\cN \to \bR$ is modular if there is a value $\ell_u$ for every $u \in \ell$ such that $\ell(S) = \sum_{u \in S} \ell_u$ for every set 
$S \subseteq \cN$.} The role of the modular function $\ell$ is to discount the benefit of adding elements. We highlight that $g-\ell$ is still 
submodular. However, it may no longer be non-negative, an assumption that is 
essential for deriving competitive algorithms with constant-factor approximation 
guarantees (for more information, see the survey by \cite{buchbinder2018submodular}).  Even though maximizing a regularized submodular function has been 
proposed in the past as a more faithful model of diverse data selection 
\cite{tschiatschek2014learning}, formal treatment of this problem has only recently been done \cite{sviridenko2017optimal,feldman2019guess,harshaw2019submodular}.

In many practical scenarios, random access to the entire data is not possible and only a small fraction of the data can be loaded to the main memory, there is only time to read the data once, and the data arrives at a very fast pace.
Furthermore, the amount of collected data is often too large to solve the optimization problem on a single machine.
Prior to this work, no streaming or distributed algorithm to solve Problem \eqref{eq:problem} was established. 
However, based on ideas from~\cite{sviridenko2017optimal,feldman2019guess}, \citet{harshaw2019submodular} proposed \AlgDistorted, an efficient offline algorithm to (approximately) solve this problem.
This algorithm iteratively and greedily finds elements that maximize a distorted function.
On the other hand, \AlgDistorted, as a centralized algorithm, requires a memory that grows linearly with the size of the data, and it needs to make multiple passes ($\Theta(n)$ in the worst case) over the data; therefore it fails to satisfy the above mentioned requirements of modern applications.
Indeed, with the unprecedented increase in the data size in 
recent years, scalable data summarization has gained a lot of attention with far-reaching 
applications, including brain parcellation from fMRI data \cite{salehi2017submodular}, interpreting neural networks 
\cite{elenberg2017streaming}, selecting panels of genomics assays \cite{wei2016choosing}, video \cite{gygli2015video}, image \cite{TIWB14,tschiatschek2014learning}, and text 
\cite{LB11,kirchhoff2014submodularity} summarization, and sparse feature selections \cite{elenberg2018restricted,das2011submodular}, to name a few.
In this paper, we propose \textbf{scalable methods} (in both streaming and distributed settings) for maximizing a \textbf{regularized} submodular function.
In the following, we briefly explain our main theoretical results. 

	\textbf{Our Results.} 
In \cref{sec:streaming}, we introduce \AlgStream, the first one-pass streaming algorithm for maximizing a regularized submodular function subject to a $k$-cardinality constraint.\footnote{Technically, this algorithm is a semi-streaming, as its space complexity is nearly linear in the size of the solution rather than being poly-logarithmic in it as is required for a streaming algorithm. As this is unavoidable for algorithms that are required to output the solution itself (rather than just estimate its value), we ignore the distinction between the two types of algorithms in this paper and refer to semi-streaming algorithms as streaming algorithms.}
Theorem~\ref{thm:result} guarantees the performance of \AlgThreshold.
\begin{theorem} \label{thm:result}
	For every $\eps , r > 0$, \AlgThreshold  produces a set $S \subseteq \cN$ of size at most $k$ for Problem~\eqref{eq:problem},  obeying $g(S) - \ell(S) \geq \max_{T \subseteq \cN, |T| \leq k} [(h(r) - \eps) \cdot g(T) - r \cdot \ell(T)]$, where $h(r) = \frac{2r + 1 - \sqrt{4r^2 + 1}}{2}$.
\end{theorem}
We should point out that previous studies of Problem~\eqref{eq:problem}, for various theoretical and practical reasons, have only focused on the case in which $r = 1$ and $T$ is the set $\Optset$ of size at most $k$ maximizing $g(T) - \ell(T)$ \cite{sviridenko2017optimal,feldman2019guess,harshaw2019submodular}. In this case, we get the following corollary from the result of Theorem~\ref{thm:result}.
\begin{corollary} \label{cor:r_1}
	For every $\eps > 0$, \AlgThreshold produces a set $S \subseteq \cN$ of size at most $k$ for Problem~\eqref{eq:problem} obeying $g(S) - \ell(S) \geq (\phi^{-2} - \eps) \cdot g(\Optset) - \ell(\Optset)$, where $\phi$ is the golden ratio (and thus, $\phi^{-2} \approx 0.382$).
\end{corollary}

In \cref{sec:distributed}, we develop \AlgDistributed, the first  distributed algorithm for regularized submodular maximization in a MapReduce model of computation.
This algorithm allows us to distribute data across several machines and use their combined computational resources.
Interestingly, even for the classic case of an unregulated monotone submodular function, our distributed algorithm improves over space and communication complexity of the existing work \cite{barbosa2016new} by a factor of $\Theta(1/\epsilon)$. 
The approximation guarantee of \AlgDistributed is given in \cref{thm:distributed}.

\newcommand{\ThmDistributed}[1][]{\AlgDistributed (Algorithm~\ref{alg:distributed}) returns a set $D \subseteq \cN$ of size at most $k$ such that
	\ifthenelse{\isempty{#1}}{\[\bE	[g(D) - \ell(D)] \geq (1 - \eps)  \left[(1 - e^{-1})  \cdot g(\Optset) - \ell(\Optset) \right] .
	\]}
	{	\[\frac{\bE	[g(D) - \ell(D)]}{1 - \eps} \geq (1 - e^{-1})  \cdot g(\Optset) - \ell(\Optset) .
	\]}
}

\begin{theorem} \label{thm:distributed}
\ThmDistributed[*]
\end{theorem}
Finally, as our algorithms  can efficiently find diverse elements from massive datasets,
in \cref{sec:mode_finding,sec:results}, we explore the power of the regularized submodular maximization approach and our algorithms in several real-world applications through an extensive set of experiments.

	\section{Related Work} \label{sec:related}

Finding the optimal solution for a submodular maximization problem is 
computationally hard even in the absence of a constraint~\cite{feige2011maximizing}. Nevertheless, the well-known result of \citet{nemhauser1978analysis} showed that the classical greedy algorithm obtains a $(1 - 1/e)$-approximation  for maximizing a non-negative and monotone
submodular function subject to a cardinality constraint, which is known to be optimal~\cite{nemhauser1978best}.
However, when the objective function is non-monotone or the constraint is more complex, the vanilla greedy algorithm may perform much worse.
An extensive line of research has lead to the development of algorithms for handling non-monotone submodular objectives subject to more complicated constraints (see, e.g., \cite{LMNS10,FNSW11,badanidiyuru2014fast,mirzasoleiman2016fast,FHK17}).
We should note that, up until very recently, all the existing works  required the 
objective function to take only non-negative values, an assumption that may not 
hold in many applications \cite{harshaw2019submodular}.

The first work to handle submodular objective functions that might take negative values is the work of \citet{sviridenko2017optimal}, which studies the maximization of submodular function that can be decomposed as a sum $g + c$, where $g$ is a non-negative
monotone submodular function and $c$ is an (arbitrary) modular function.
For this problem, \citet{sviridenko2017optimal} gave two randomized polynomial-time algorithms which produce a set $S$ that roughly obeys $g(S) +  c(S) \geq  (1 - 1/e) \cdot g(\Optset) + c(\Optset)$, where $\Optset$ is the optimal set. 
Both algorithms of \citet{sviridenko2017optimal} are mainly of theoretical 
interest, as their computational complexity is quite prohibitive. 
\citet{feldman2019guess} reconsidered one of these algorithms, and showed that 
one of the costly steps in it (namely, a step in which the algorithm guesses the 
value of $c(\Optset)$) can be avoided using a surrogate objective that varies with 
time. Nevertheless, the algorithm of \cite{feldman2019guess} remains quite 
involved as it optimizes a fractional version of the problem---which is necessary 
for allowing it to handle various kinds of complex constraints. 
\citet{harshaw2019submodular} showed that in the case of a  cardinality 
constraint (and a non-negative $c$) much of this complexity can be avoided, 
yielding the first practical algorithm \AlgDistorted. They also extended their 
results to $\gamma$-weakly submodular functions and the unconstrained setting.

Due to the massive volume of the current data sets, scalable methods have gained 
a lot of interest in machine learning applications. One appealing approach towards this goal is to 
design streaming algorithms. \citet{badanidiyuru2014streaming} were the first to  
consider a single-pass streaming algorithm for maximizing a monotone 
submodular function under a cardinality constraint. Their result was later 
improved and extended to non-monotone 
functions~\cite{alaluf2019making,ene2019optimal,kazemi2019submodular} and 
subject to more  involved 
constraints~\cite{buchbinder2015online,chakrabarti2015submodular,chekuri2015streaming,feldman2018do}. Another scalable approach is the development  of distributed algorithms through the MapReduce framework where the data is split amongst several machines and processed in parallel 
\citep{kumar2015fast, mirzasoleiman16distributed, barbosa2015power, 
mirrokni2015randomized, barbosa2016new, liu2018submodular}.

In the context of discrete probabilistic models, it is well-known that strongly 
Rayleigh (SR) measures \cite{borcea2009negative} (including determinantal point
processes \cite{kulesza2012determinantal}) or the  more general class of strongly 
log-concave (SLC) distributions \cite{gurvits2009polynomial}   provide strong 
negative dependence among sampling items. Although \citet{gotovos2019strong} 
recently showed that strong log-concavity does not imply log-submodularity,  
\citet{robinson2019flexible}  argued that the logarithm of an SLC distribution 
enjoys a variant of approximate submodularity. We build on this result and derive 
a slight improvement, along with corresponding guarantees for streaming and 
distributed solutions.

	\section{Streaming Algorithm} \label{sec:streaming}

In this section, we present our proposed streaming algorithm for Problem~\eqref{eq:problem}.
To explain our algorithm, let us first define $T$ to be a subset of $\cN$ of size at most $k$ such that
\[ T \in \argmax_{S \subseteq \cN, |S| \leq k} [(h(r) - \eps) \cdot g(T) - r \cdot \ell(T)] \enspace,\]
where $r$ is some positive real value to be discussed later, and $h(r) = \frac{2r + 1 - \sqrt{4r^2 + 1}}{2}$.
A basic version of our proposed algorithm, named \AlgThreshold, is given as Algorithm~\ref{alg:threshold-first}. 
We note that this algorithm guesses, in the first step, a value $\tau > 0$ which obeys $k\tau \leq h(r) \cdot g(T) - r \cdot \ell(T) \leq (1 + \eps)k\tau$. 
In \cref{alg:threshold-first}, to avoid unnecessary technicalities, we simply assume that the algorithm can guess such a value for $\tau$ based on some oracle.
In Appendix~\ref{sec:tau}, we explain how a technique from~\cite{badanidiyuru2014streaming} can be used for that purpose at the cost of increasing the space complexity of the algorithm by a factor of $O(\eps^{-1}(\log k + \log r^{-1}))$.
\cref{alg:threshold-first} starts with an empty set $S$. 
While the data stream is not empty yet and the size of set $S$ is still smaller than $k$, for every incoming element $u$, the value of $g(u \mid S) - \alpha(r) \cdot \ell(\{u\})$ is calculated, where we define $\alpha(r) = \frac{2r + 1 + \sqrt{4r^2 + 1}}{2}$. If this value is at least $\tau$, then $u$ is added to $S$ by the algorithm.
The theoretical guarantee of Algorithm~\ref{alg:threshold-first} is provided in Theorem~\ref{thm:result}, and the proof of this theorem is given in \cref{sec:thm_proof}.

\begin{algorithm2e}[htb!]
	\caption{\AlgThreshold} \label{alg:threshold-first}
	Guess a value $\tau$ such that $k\tau \leq h(r) \cdot g(T) - r \cdot \ell(T) \leq (1 + \eps)k\tau$.\\
	Let $\alpha(r) \gets \frac{2r + 1 + \sqrt{4r^2 + 1}}{2}$.\\
	Let $S \gets \varnothing$.\\
	\While{$|S| < k$ and there are more elements}
	{
		Let $u$ be the next elements in the stream.\\
		\If{$g(u \mid S) - \alpha(r) \cdot \ell(\{u\}) \geq \tau$}
		{
			Add $u$ to the set $S$.
		}
	}
	\Return{the better solution among $S$ and $\varnothing$}.
\end{algorithm2e}

\subsection{How to Choose a Good Value of $r$?}

In this section, we study the effect of the parameter $r$ on the performance of Algorithm~\ref{alg:threshold-first} under different settings. 
First note that the bound given by Corollary~\ref{cor:r_1} reduces to a trivial lower bound of $0$ when $\phi^{-2} \cdot g(\Optset) \leq \ell(\Optset)$. A similar phenomenon happens for the bound of \cite[Theorem~3]{harshaw2019submodular} when $(1 - e^{-1}) \cdot g(\Optset) \leq \ell(\Optset)$. Namely, in this regimen their bound (i.e., $(1 - e^{-1}) \cdot g(\Optset) - \ell(\Optset)$) becomes trivial. We now explain how a carefully chosen value for $r$ can be used to prevent this issue.

For a set $S$, let $\beta_S$ denote the ratio between the utility of $S$ and its linear cost, i.e., $\beta_S = \frac{g(S) - \ell(S)}{\ell(S)}$. Using this terminology, we get that the guarantees of Corollary~\eqref{cor:r_1} and \cite[Theorem~3]{harshaw2019submodular} become trivial when $\beta_{\Optset} \leq \phi^2 - 1 = \phi$ and $\beta_{\Optset} \leq \nicefrac{1}{(e - 1)}$, respectively.
In Corollary~\ref{cor:modified-bound}, we show that by knowing the value of $\beta_{\Optset}$ we can find a value for $r$ in Algorithm~\ref{alg:threshold-first} which makes Theorem~\ref{thm:result} yield the strongest guarantee for \eqref{eq:problem}, and moreover, this guarantee is non-trivial as long as $\beta_{\Optset} > 0$ (if $\beta_{\Optset} \leq 0$, then the empty set is a trivial optimal solution).\footnote{The value $\beta_{\Optset}$ is undefined when $\ell(\Optset) = 0$. We implicitly assume in this section that this does not happen. However, if this assumption is invalid for the input, one can handle the case of $\ell(\Optset) = 0$ by simply dismissing all the elements whose linear cost is positive and then using an algorithm for unregulated submodular maximization on the remaining elements.}

\newcommand{\CorModifiedBound}{Assume the value of $\beta_{\Optset}$ is given, where $\Optset$ is the optimal solution of Problem~\eqref{eq:problem}. 
	Setting $r = r_{\Optset} = \frac{\beta_{\Optset}}{2 \sqrt{1 + 2\beta_{\Optset}}}$ makes Algorithm~\ref{alg:threshold-first} return a solution $S$ with the guarantee
	\begin{align*}
	g(S) - \ell(S) \geq  \left(\frac{1 + \beta_{\Optset} - \sqrt{1 + 2\beta_{\Optset}}}{2\beta_{\Optset}} - \epsilon' \right) \cdot 
	(g(OPT) - \ell(OPT))\enspace,
	\end{align*}
	\vspace{-5pt}
	where $\epsilon' = \epsilon \cdot (1 + 1 / \beta_{\Optset})$.}

\begin{corollary} \label{cor:modified-bound}
\CorModifiedBound
\end{corollary}

\begin{proof}
	First, let us define $T^*_r = \argmax_{T \in \cN, |T| \leq k} [(h(r) - \epsilon) \cdot g(T) - r \cdot \ell(T)] $.  From Theorem~\ref{thm:result} and the definition of $T^*_r$, we have
	\[g(S) - \ell(S) \geq (h(r) - \varepsilon) \cdot g(T^*_r) - r \cdot \ell(T^*) \geq  (h(r) - \varepsilon) \cdot g(\Optset) - r \cdot \ell(\Optset)  \enspace. \]
	Furthermore, from the definition of  $\beta_{\Optset}$, we have
	\begin{align*}
	(h(r) - \epsilon) \cdot g(\Optset) - r \cdot \ell(\Optset) ={}& \frac{(h(r) - \epsilon) \cdot g(\Optset) - r \cdot \ell(\Optset)}{g(\Optset) - \ell(\Optset)}\cdot (g(\Optset) - \ell(\Optset) )\\ ={} & \frac{(h(r) - \epsilon) \cdot (1 + \beta_{\Optset}) - r}{\beta_{\Optset}}\cdot (g(\Optset) - \ell(\Optset) ) \enspace.
	\end{align*}
	It can be verified that $r_{\Optset}$ is the value that maximizes the above expression, and plugging this value into the expression proves the corollary.
\end{proof}

From the definition of $\beta_{\Optset}$, it can be observed that for larger values of $\beta_{\Optset}$ the effect of the modular cost function over the utility diminishes and $g - \ell$ gets closer to a monotone and non-negative submodular function. 
At the same time, from Corollary~\ref{cor:modified-bound} we see that for large values of $\beta_{\Optset}$ the approximation factor approaches $\nicefrac{1}{2}$. \citet{norouzifard2018beyond} gave evidence that the last approximation ratio is optimal when the objective function is indeed non-negative and monotone,\footnote{Formally, they showed that no streaming algorithm can produce a solution with an approximation guarantee better than $\nicefrac{1}{2}$ for such objective functions using $o(n / k)$ memory, as long as it queries the value of the submodular function only for feasible sets.}
 which could be an indicator for the optimality of our streaming algorithm for \eqref{eq:problem}. Indeed, we \textbf{conjecture} that \AlgThreshold, when choosing $r$ based on $\beta_{\Optset}$ as in Corollary~\ref{cor:modified-bound}, achieves the best possible guarantee for Problem~\eqref{eq:problem} in the streaming setting.

In order to apply the result of  Corollary~\ref{cor:modified-bound} to obtain the strongest guarantee for Problem~\eqref{eq:problem}, we need to have access to the set $\Optset$ (and consequently $\beta_{\Optset}$ and $r_{\Optset}$); but, unfortunately, none of these is known a priori. Next, we propose an efficient approach that enables us to find an accurate enough estimate of $r_{\Optset}$.
Let $\zeta_{\Optset} = \frac{1 + \beta_{\Optset} - \sqrt{1 + 2\beta_{\Optset}}}{2\beta_{\Optset}} $ be the approximation ratio that can be obtained for the unknown value of $\beta_{\Optset} $ via Corollary~\ref{cor:modified-bound} (except for the $\eps'$ error term). 
This definition implies that we always have $0 \leq \zeta_{\Optset} < \nicefrac{1}{2}$. 
Thus, we can find an accurate guess for $\zeta_{\Optset}$ by dividing the interval $[\eps, \nicefrac{1}{2})$ to small intervals (values of $\zeta_{\Optset}$ below $\epsilon < \eps'$ are not of interest because Corollary~\ref{cor:modified-bound} gives a trivial guarantee for them).
Moreover, given a guess for $\zeta_{\Optset}$, we can calculate the corresponding values of $\beta_{\Optset}$ and $r_{\Optset}$.
The full version of the proposed algorithm, named \AlgStream, is based on this idea. Its pseudocode is given as Algorithm~\ref{alg:threshold}, and assumes that $\delta > 0$ is an accuracy parameter. 
 
\begin{algorithm2e}[htb!]
	\DontPrintSemicolon
	\caption{\AlgStream} \label{alg:threshold}
	$\Lambda \gets  \{ \eps (1+\delta)^{i} \mid 0 \leq i \leq \lfloor \log_{1+\delta} (\nicefrac{1}{(2\eps)})\rfloor \}$.\\
	\For{every $\zeta \in \Lambda$ \textbf{in parallel}}{
	Calculate $\beta \gets \frac{4 \zeta}{(1 - 2 \zeta)^2}$.\\
	Calculate $r \gets \frac{\beta}{2 \sqrt{1 + 2\beta}}$. \tcp*{This is the formula from Corollary~\ref{cor:modified-bound}.}
	Run \AlgThreshold(Algorithm~\ref{alg:threshold-first}) with $r$.
}
	\Return{the best among the solutions found by all the copies of \AlgThreshold executed}.
\end{algorithm2e}

One can see that the value of $r$ passed to every copy of \AlgThreshold by \AlgStream is at least $2\eps$, which implies that the number of elements kept by each such copy is at most $O(\eps^{-1}(\log k + \log \eps^{-1}))$. Furthermore, the number of elements kept by \AlgStream is larger than that by a factor of at most $1 + \log_{1+\delta} (\nicefrac{1}{(2\eps)}) = O(\delta^{-1} \cdot \log (\eps)^{-1})$. The following observation studies the approximation guarantee of \AlgStream.
\begin{lemma} \label{lemma:zeta-approx}
	Despite not assuming access to $\beta_{\Optset}$, \AlgStream  outputs a set $S$ obeying
	\begin{align*}
		g(S) - \ell(S) \geq  \left((1 - \delta') \cdot \zeta_{\Optset} - \epsilon' \right) \cdot  
		(g(OPT) - \ell(OPT))\enspace,
	\end{align*}
	where $\delta' = \delta/2$ and  $\epsilon'  = \frac{\epsilon}{2 \zeta_{\Optset}} $.
\end{lemma}

\begin{proof}
	For values of $\zeta_{\Optset} < \epsilon$, the right-hand side of the lower bounded provided by the lemma is negative, and thus, it gives a trivial lower bound (note that $\eps' \geq \eps$ and $\delta' > 0$ since $\zeta_{\Optset}$ is always smaller than $\nicefrac{1}{2}$).
	For this reason, in the rest of proof, we assume $\zeta_{\Optset} \geq \epsilon$.
	First, note that there must be a value $\zeta \in \Lambda$ such that $\zeta \leq  \zeta_{\Optset} < (1+\delta) \cdot \zeta$, and
	let us denote $\omega = \frac{\zeta_{\Optset}}{\zeta}$. 
	It is clear that $1 \leq \omega < 1 + \delta$. Moreover, using the definition of $\omega$ we get that the value of $\beta$ corresponding to $\zeta$ is $\beta = \frac{4 \omega \zeta_{\Optset}}{(\omega - 2 \zeta_{\Optset})^2}$, and the value of $r$ corresponding to this $\zeta$ is
	\begin{align*}
	r
	={} &
	\frac{\beta}{2\sqrt{1 + 2\beta}}
	=
	\frac{4 \omega \zeta_{\Optset} / (\omega - 2 \zeta_{\Optset})^2}{2\sqrt{1 + 8 \omega \zeta_{\Optset} / (\omega - 2 \zeta_{\Optset})^2}}\\
	={} &
	\frac{4 \omega \zeta_{\Optset}}{2(\omega - 2 \zeta_{\Optset}) \cdot \sqrt{(\omega - 2 \zeta_{\Optset})^2 + 8 \omega \zeta_{\Optset}}}
	=
	\frac{4 \omega \zeta_{\Optset}}{2(\omega - 2 \zeta_{\Optset}) \cdot \sqrt{(\omega + 2 \zeta_{\Optset})^2}}
	=
	\frac{2 \omega \zeta_{\Optset}}{\omega^2 - 4 \zeta_{\Optset}^2}
	\enspace.
	\end{align*}
	To calculate the value of $h(r)$ corresponding to this value of $r$, we note that:
	\begin{align*}
	\sqrt{4r^2 + 1} 
	={} & \sqrt{4\left(\frac{2 \omega \zeta_{\Optset}}{\omega^2 - 4 \zeta_{\Optset}^2}\right)^2 + 1}  = \frac{\omega^2 +4 \zeta_{\Optset}^2}{\omega^2 - 4 \zeta_{\Optset}^2} 
	\enspace.
	\end{align*}
	If we plug this equality into the definition of $h(r)$, we get
	\begin{align*}
	h(r)  
	={} & \frac{2r + 1 - \sqrt{4r^2 + 1}}{2} = \frac{1}{2} \cdot \left[   1 + 	\frac{4 \omega \zeta_{\Optset}}{\omega^2 - 4 \zeta_{\Optset}^2} -  \frac{\omega^2 +4 \zeta_{\Optset}^2}{\omega^2 - 4 \zeta_{\Optset}^2}   \right] = \frac{2 \omega \zeta_{\Optset} - 4\zeta_{\Optset}^2}{\omega^2 - 4 \zeta_{\Optset}^2}  = \frac{2 \zeta_{\Optset}}{\omega + 2 \zeta_{\Optset}}
	\enspace.
	\end{align*}
	We are now ready to plug the  calculated value of $r$ into Theorem~\ref{thm:result}, which yields that the output set $S'$ of the instance of \AlgThreshold initialized with this value of $r$ obeys
	\begin{align} \label{eq:S_prime}
	g(S') - \ell(S')
	\geq{} &
	(h(r) - \eps) \cdot g(\Optset) - r \cdot \ell(\Optset)\\ \nonumber
	={} &
	\frac{(h(r) - \epsilon) \cdot g(\Optset) - r \cdot \ell(\Optset)}{g(\Optset) - \ell(\Optset)}\cdot (g(\Optset) - \ell(\Optset) )\\ 	={} &
	\frac{(h(r) - \epsilon) \cdot (1 + \beta_{\Optset}) - r}{\beta_{\Optset}}\cdot (g(\Optset) - \ell(\Optset) )
	\enspace,
	\end{align}
	where the last equality follows from the definition of $\beta_{\Optset}$. Let us now lower bound the coefficient of $g(\Optset) - \ell(\Optset)$ in the rightmost hand side of the last equality.
	Recalling that $\beta_{\Optset} = \frac{4 \zeta_{\Optset}}{(1 - 2 \zeta_{\Optset})^2}$, we get $\frac{1 + \beta_{\Optset}}{\beta_{\Optset}} = \frac{1 + 4 \zeta_{\Optset}^2}{ 4 \zeta_{\Optset}}$. Thus, the above mentioned coefficient can be written as
	\begin{align*}
	\frac{(h(r) - \eps) \cdot (1 + \beta_{\Optset}) - r}{\beta_{\Optset}}
	={} &
	\frac{1 + \beta_{\Optset}}{\beta_{\Optset}} \cdot h(r) - \frac{r}{\beta_{\Optset}} - 	\frac{(1 + \beta_{\Optset}) \cdot \epsilon}{\beta_{\Optset}} \\
	={} &
	\frac{1 + 4 \zeta_{\Optset}^2}{ 4 \zeta_{\Optset}} \cdot \frac{2 \zeta_{\Optset}}{\omega + 2 \zeta_{\Optset}} - \frac{\omega \cdot (1 - 2 \zeta_{\Optset})^2}{2 \cdot (\omega^2 - 4 \zeta_{\Optset}^2)} -  \frac{(1 + 4 \zeta_{\Optset}^2) \cdot \epsilon}{ 4 \zeta_{\Optset}}  \\
	={} &  \frac{2 \omega \zeta_{\Optset} - \zeta_{\Optset} - 4 \zeta_{\Optset}^3}{\omega^2 - 4 \zeta_{\Optset}^2}  -  \frac{(1 + 4 \zeta_{\Optset}^2) \cdot \epsilon}{ 4 \zeta_{\Optset}}  \\
	={} & \left(1 - \frac{(\omega - 1)^2}{\omega^2 - 4 \zeta_{\Optset}^2}\right) \cdot \zeta_{\Optset} -  \frac{(1 + 4 \zeta_{\Optset}^2) \cdot \epsilon}{ 4 \zeta_{\Optset}} \enspace.
	\end{align*}
	We can observe that the coefficient of $\zeta_{\Optset}$ on the rightmost side is a decreasing function of $\omega$ for $\omega \geq 4\zeta_{\Optset}^2$. Together with the facts that $\zeta_{\Optset} < 1/2$ and $\omega \geq 1$, this implies
	\begin{align*}
	\frac{(h(r) - \eps) \cdot (1 + \beta_{\Optset}) - r}{\beta_{\Optset}}
	\geq{} &  \left(1 - \frac{\delta^2}{(1 + \delta)^2 - 4 \zeta_{\Optset}^2}\right) \cdot \zeta_{\Optset} - \frac{\epsilon}{2 \zeta_{\Optset}}\\
	\geq{} &
	\left(1 - \frac{\delta^2}{2\delta}\right) \cdot \zeta_{\Optset} - \frac{\epsilon}{2 \zeta_{\Optset}}
	=
	\left(1 - \frac{\delta}{2}\right) \cdot \zeta_{\Optset} - \frac{\epsilon}{2 \zeta_{\Optset}}
	\enspace.
	\end{align*}
	Plugging this inequality into \cref{eq:S_prime}, we get that the set $S'$ produced by \AlgThreshold for the above value of $r$ has at least the value guaranteed by the lemma for the output set $S$ of \AlgStream. The lemma now follows since the set $S$ is chosen as the best set among multiple options including $S'$.
\end{proof}

\subsection{Proof of Theorem~\ref{thm:result}} \label{sec:thm_proof}

The proof of Theorem~\ref{thm:result} is based on handling two cases. The first case is when the size of the solution $S$ of Algorithm~\ref{alg:threshold-first} reaches the maximum possible size $k$. In this case, we prove the approximation ratio of \cref{alg:threshold-first} by noting that $S$ contains many elements in this case, and each one of these elements had a large marginal contribution upon arrival due to the condition used by Algorithm~\ref{alg:threshold-first} to decide whether to add an arriving element.
The second case is when $S$ does not reach its maximum possible size $k$. In this case, we know, by the submodularity of $g$, that every element of $\Optset$ that was not added to $S$ has a small marginal contribution with respect to the final set $S$. This allows us to upper bound the additional value that could have been contributed by these elements.

In the rest of this section, we provide detailed proof of Theorem~\ref{thm:result}. Let $T$ be the subset of $\cN$ of size at most $k$ maximizing $(h(r) - \eps) \cdot g(T) - r \cdot \ell(T)$ among all such subsets. If $h(r) \cdot g(T) - r \cdot \ell(T) \leq 0$, then the empty set is a solution set obeying all the requirements of Theorem~\ref{thm:result}. Thus, in the rest of this section, we assume $h(r) \cdot g(T) - r \cdot \ell(T) > 0$, which implies that the value $\tau$ guessed by \cref{alg:threshold-first} is positive.

The following two lemmata prove together that Algorithm~\ref{alg:threshold-first} has the approximation guarantee of Theorem~\ref{thm:result}. The first of these lemmata handles the case in which the size of the solution $S$ of Algorithm~\ref{alg:threshold-first} reaches the maximum possible size $k$. 
In the proofs of both lemmata $u_i$ denotes the $i$-th element added to $S$ by Algorithm~\ref{alg:threshold-first}. 

\newcommand{\LemmaFull}{If $|S| = k$, then $g(S) - \ell(S) \geq (h(r) - \eps) \cdot g(T) - r \cdot \ell(T)$.}
\begin{lemma} \label{lem:full}
	\LemmaFull
\end{lemma}
\begin{proof}
	Observe that
	\begin{align*}
	g(S) - \alpha(r) \cdot \ell(S)
	={} &
	\sum_{i = 1}^k [g(u_i \mid \{u_1, u_2, \dotsc, u_{i - 1}\}) - \alpha(r) \cdot \ell(\{u\})]
	\geq
	k\tau\\
	\geq{} &
	\frac{h(r) \cdot g(T) - r \cdot \ell(T)}{1 + \eps}
	\geq
	(1 - \eps) \cdot h(r) \cdot g(T) - r \cdot \ell(T)
	\enspace,
	\end{align*}
	where the first inequality holds since Algorithm~\ref{alg:threshold-first} chose to add $u_i$ to $S$, and the set $S$ at that time was equal to $\{u_1, u_2, \dotsc, u_{i - 1}\}$.

	We now make two observations. First, we observe that
	\[
	\alpha(r)
	=
	\frac{2r + 1 + \sqrt{4r^2 + 1}}{2}
	\geq
	\frac{1 + \sqrt{1}}{2}
	=
	1
	\enspace,
	\]
	and second, we observe that $h(r) \leq 1/2$ because
	\[
	h(r) \leq 1/2
	\iff
	\frac{2r + 1 - \sqrt{4r^2 + 1}}{2} \leq 1/2
	\iff
	2r \leq \sqrt{4r^2 + 1}
	\iff
	4r^2 \leq 4r^2 + 1
	\enspace.
	\]
	Using, these observations and the above inequality, we now get
	\[
	g(S) - \ell(S)
	\geq
	g(S) - \alpha(r) \cdot \ell(S)
	\geq
	(1 - \eps) \cdot h(r) \cdot g(T) - r \cdot \ell(T)
	\geq
	(h(r) - \eps) \cdot g(T) - r \cdot \ell(T)
	\enspace.
	\qedhere
	\]
\end{proof}

The following lemma proves the approximation ratio of Algorithm~\ref{alg:threshold-first} for the case in which the solution set $S$ does not reach its maximum allowed size $k$ before the stream ends.

\newcommand{\LemmaNotFull}{If $|S| < k$, then $g(S) - \ell(S) \geq (h(r) - \eps) \cdot g(T) - r \cdot \ell(T)$.}
\begin{lemma}  \label{lemma:not-full}
	\LemmaNotFull
\end{lemma}
\begin{proof}
	Consider an arbitrary element $u \in \Optset \setminus S$. Since $|S| < k$, the fact that $u$ was not added to $S$ implies
	\[
	g(u \mid S') - \alpha(r) \cdot \ell(\{u\}) < \tau
	\enspace,
	\]
	where $S'$ is the set $S$ at the time in which $u$ arrived. By the submodularity of $g$, we also get
	\[
	g(u \mid S) - \alpha(r) \cdot \ell(\{u\}) < \tau
	\enspace.
	\]
	
	Adding the last inequality over all elements $u \in T \setminus S$ implies
	\begin{align*}
	g(T) - g(S) - \alpha(r) \cdot \ell(T)
	\leq{} &
	g(T \mid S) - \alpha(r) \cdot \ell(T)
	\leq
	\sum_{u \in T \setminus S} [g(u \mid S) - \alpha(r) \cdot \ell(\{u\})]\\
	<{} &
	k\tau
	\leq
	h(r) \cdot g(T) - r \cdot \ell(T)
	\enspace,
	\end{align*}
	where the first inequality follows from the monotonicity of $g$, and the second inequality holds due to the submodularity of $g$ and the non-negativity of $\ell$. Rearranging this inequality yields
	\begin{equation} \label{eq:over_OPT_bound}
	(1 - h(r)) \cdot g(T) + (r - \alpha(r)) \cdot \ell(T)
	<
	g(S)
	\enspace.
	\end{equation}
	Recall that $\tau > 0$. Thus, using the same argument used in the proof of Lemma~\ref{lem:full}, we get
	\[
	g(S) - \alpha(r) \cdot \ell(S)
	=
	\sum_{i = 1}^{|S|} [g(u_i \mid \{u_1, u_2, \dotsc, u_{i - 1}\}) - \alpha(r) \cdot \ell(\{u\})]
	\geq
	|S|\tau
	\geq
	0
	\enspace.
	\]
	Adding a $1/\alpha(r)$ fraction of this equation to a $1 - 1 / \alpha(r)$ fraction of Equation~\eqref{eq:over_OPT_bound} yields
	\[
	g(S) - \ell(S)
	>
	(1 - 1/\alpha(r))(1 - h(r)) \cdot g(T) + (1 - 1/\alpha(r))(r - \alpha(r)) \cdot \ell(T)
	\enspace.
	\]
	The following two calculations now complete the proof of the lemma (since $\eps \cdot g(T)$ is non-negative).
	\begin{align*}
	(1 - 1/\alpha(r))(1 - h(r))
	={} &
	\left(1 - \frac{2}{2r + 1 + \sqrt{4r^2 + 1}}\right)\left(1 - \frac{2r + 1 - \sqrt{4r^2 + 1}}{2}\right)\\
	={} &
	\frac{2r - 1 + \sqrt{4r^2 + 1}}{2r + 1 + \sqrt{4r^2 + 1}} \cdot \frac{1 - 2r + \sqrt{4r^2 + 1}}{2}
	=
	\frac{4r^2 + 1 - (2r - 1)^2}{2(2r + 1 + \sqrt{4r^2 + 1})}\\
	={} &
	\frac{(2r + 1)^2 - (4r^2 + 1)}{2(2r + 1 + \sqrt{4r^2 + 1})}
	=
	\frac{2r + 1 + \sqrt{4r^2 + 1}}{2r + 1 + \sqrt{4r^2 + 1}} \cdot \frac{2r + 1 - \sqrt{4r^2 + 1}}{2}
	=
	h(r)
	\enspace,
	\end{align*}
	and
	\begin{align*}
	(1 - 1/\alpha(r))(r - \alpha(r))&
	=
	\left(1 - \frac{2}{2r + 1 + \sqrt{4r^2 + 1}}\right)\left(r - \frac{2r + 1 + \sqrt{4r^2 + 1}}{2}\right)\\
	={} &
	- \frac{2r - 1 + \sqrt{4r^2 + 1}}{2r + 1 + \sqrt{4r^2 + 1}} \cdot \frac{1 + \sqrt{4r^2 + 1}}{2}
	=
	- \frac{2r - 1 + (4r^2 + 1) + 2r \cdot \sqrt{4r^2 + 1}}{2[2r + 1 + \sqrt{4r^2 + 1}]}\\
	={} &
	- \frac{2r \cdot [1 + 2r + \sqrt{4r^2 + 1}]}{2[2r + 1 + \sqrt{4r^2 + 1}]}
	=
	-r
	\enspace.
	\qedhere
	\end{align*}
\end{proof}

	\section{Distributed Algorithm} \label{sec:distributed}
The exponential growth of data makes it difficult to process or even store the entire data on a single machine.
For this reason, there is an urgent need to develop distributed or parallel computing methods to process massive datasets.
Distributed algorithms in a Map-Reduce model have shown promising results in several problems related to submodular maximization \cite{mirzasoleiman2013distributed,mirrokni2015randomized,barbosa2015power,barbosa2016new,mitrovic2018data,kazemi2018scalable}. 
In this section we present a distributed solution for Problem~\eqref{eq:problem}, named \AlgDistributed, which appears as \cref{alg:distributed}.
Our algorithm uses \AlgDistorted proposed by \citet{harshaw2019submodular} as a subroutine.

Out distributed solution is based on the framework suggested by \citet{barbosa2016new} for converting greedy-like sequential algorithms into a distributed algorithm. However, its analysis is based on a generalization of ideas from \cite{barbosa2015power} rather than being a direct adaptation of the analysis given by~\cite{barbosa2016new}. This allows us to get an algorithm which uses asymptotically the same number of computational rounds as the algorithm of~\cite{barbosa2016new}, but does not require to keep multiple copies of the data as is necessary for the last algorithm.\footnote{Our technique can be used to get the same improvement for the setting of~\cite{barbosa2016new}. However, as this is not the main subject of this paper, we omit the details.} We would also like  to point out that \citet{barbosa2016new} have proposed a variant of their algorithm that avoids data replication, but it does so at the cost of increasing the number of rounds from $\Theta(1/\epsilon)$ to $\Theta(1/\epsilon^2)$.

\AlgDistributed is a distributed algorithm within a Map-Reduce framework using $\lceil 1/\eps \rceil$ rounds of computation, where $\eps \in (0, 1/2]$ is a parameter controlling the quality of the output produced by the algorithm---for simplicity, we assume that $1 / \eps$ is an integer from this point on. In the first round of computation, \AlgDistributed randomly distributes the elements among $m$ machines by independently sending each element $u \in \cN$ to a uniformly random machine. Each machine $i$ then runs \AlgDistorted on its data and forwards the resulting solution $S_{1, i}$ to all other machines (in general, we denote by $S_{r, i}$ the solution calculated by machine $i$ in round $r$). The next rounds repeat this operation, except that the data of each machine now includes both: (i) elements sent to this machine during the random partition and (ii) the elements that belong to any solutions calculated (by any machine) during the previous rounds. At the end of the last round, machine number $1$ outputs the final solution, which is the best solution among the solution computed by this machine in the last round and the solutions computed by all the machines in the previous rounds.
Theorem~\ref{thm:distributed} analyzes the approximation guarantee of \AlgDistributed.
In the rest of this section, we prove Theorem~\ref{thm:distributed}.

\begin{algorithm2e}[htb!]
\DontPrintSemicolon
	\caption{\AlgDistributed} \label{alg:distributed}
	\For{$r = 1$ \KwTo $\lceil \eps^{-1} \rceil$}
	{
		\For{each $u \in \cN$}
		{
			Assign element $u$ to a machine chosen uniformly at random (and independently) among $m$ machines.\\
			Let $\cN_{r, i}$ denote the elements assigned to machine $i$ in this round.
		}
		\For{$i = 1$ \KwTo $m$}{Run \AlgDistorted on the set $\cN_{r, i} \cup (\cup_{r' = 1}^{r - 1} \cup_{i' = 1}^m S_{r', i'})$ to get the solution $S_{r, i}$ of size at most $k$
		.}
		\lIf{$r < \eps^{-1}$}{Forward the solutions $S_{r, i}$, for every integer $1 \leq i \leq m$, to all the machines.}
		\lElse{\Return{a set $D$ maximizing $g(D) - \ell(D)$ among all sets in $\{S_{r, 1}\} \cup \{S_{r', i'} \mid 1 \leq r' < r \text{ and } 1 \leq i' \leq m\}$}.}
	}
\end{algorithm2e}

\paragraph{Proof of Theorem~\ref{thm:distributed}.} 
We define the submodular function $f(S) \triangleq g(S) - \ell(S)$.  It is easy to see that $f$ is a submodular function (although it is not guaranteed to be either monotone or non-negative). The Lov\'asz extension of $f$ is the function $\lov{f}\colon [0,1]^{\cN} \to \bR$ given by
\begin{align*}
	\lov{f}(\mathbf{x})=\underset{\theta \in \mathcal{U}(0,1)}{\mathbb{E}}\left[f\left(\left\{i: x_{i} \geq \theta\right\}\right)\right] \enspace,
\end{align*} 
where $\cU(0, 1)$ is the uniform distribution within the range $[0, 1]$~\cite{lovasz1983submodular}. Note that the Lov\'asz extension of a modular set function is the natural linear extension of the function. It was also proved in~\cite{lovasz1983submodular} that the Lov\'{a}sz extension of a submodular function is convex. Finally, we need the following well-known properties of Lov\'asz extensions, which follow easily from its definition.
\begin{observation} \label{obs:basic_properties}
	For every set $S \subseteq \cN$, $\lov{f}(\mathbf{1}_S) = f(S)$. Additionally, $\lov{f}(c \cdot \mathbf{p}) \geq c \cdot \lov{f}(\mathbf{p})$ for every $c \in[0,1]$ and $\mathbf{p} \in [0, 1]^\cN$ whenever $f(\varnothing)$ is non-negative.
\end{observation}

Let us denote by $\AlgDistorted(A)$ the set produced by $\AlgDistorted$ when it is given the elements of a set $A \subseteq \cN$ as input. Using this notation, we can now state the following lemma. We omit the simple proof of this lemma, but note that it is similar to the proof of \cite[Lemma~2]{barbosa2015power}.
\begin{lemma} \label{lemma:lemma:dg-union}
	Let $A \subseteq \cN$ and $B \subseteq \cN$ be two disjoint subsets of $\cN$. Suppose that, for each element
	$u \in B$, we have $\AlgDistorted(A \cup \{u\}) = \AlgDistorted(A)$. Then, $\AlgDistorted(A \cup B) = \AlgDistorted(A)$.
\end{lemma}

We now need some additional notation. Let $S^*$ denote an optimal solution for Problem~\eqref{eq:problem}, and let $\mathcal{\cN}(1 / m)$ represent the distribution over random subsets of $\cN$ where each element is sampled independently with probability $1 / m$. To see why this distribution is important, recall that $\cN_{r, i}$ is the set of elements assigned to machine $i$ in round $i$ by the random partition, and that every element is assigned uniformly at random to one out of $m$ machines, which implies that the distribution of $\cN_{r, i}$ is identical to $\mathcal{\cN}(1 / m)$ for every two integers $1 \leq i \leq m$ and $1 \leq r \leq \eps^{-1}$. We now define for every integer $0 \leq r \leq \eps^{-1}$ the set $C_r = \cup_{r' = 1}^r \cup_{i = 1}^m S_{r', i}$ and the vector $\mathbf{p}^r \in [0,1]^{\cN}$ whose $u$-coordinate, for every $u \in \cN$, is given by
\[ p^r_u=
\begin{cases}
\Pr_{A \sim \cN(1/m)} [u \not \in C_{r - 1} \text{ and } u \in \AlgDistorted(A \cup C_{r - 1} \cup \{u\})] & \text{if } u \in S^* \enspace,  \\
0 & \text{otherwise} \enspace.
\end{cases}
\]
The next lemma proves an important property of the above vectors.
\begin{lemma} \label{lem:sum_p}
	For every element $u \in S^*$ and $0 \leq r \leq 1 / \eps$, $\Pr[u \in C_r] = \sum_{r' = 1}^r p^{r'}_u$.
\end{lemma}
\begin{proof}
	Since $u$ is assigned in round $r'$ to a single machine uniformly at random,
	\begin{align*}
		\mspace{100mu}&\mspace{-100mu}
		\Pr[u \in C_{r'} \setminus C_{r' - 1}]
		=
		\Pr[u \in \cup_{i = 1}^m S_{r', i} \setminus C_{r' - 1}]
		=
		\frac{1}{m} \sum_{i = 1}^m \Pr[u \in S_{r', i} \setminus C_{r' - 1} \mid u \in \cN_{r', i}]\\
		={} &
		\frac{1}{m} \sum_{i = 1}^m \Pr[u \not \in C_{r' - 1} \text{ and } u \in \AlgDistorted(\cN_{r', i} \cup (\cup_{r'' = 1}^{r' - 1} \cup_{i' = 1}^m S_{r'', i'})) \mid u \in \cN_{r', i}]\\
		={} &
		\frac{1}{m} \sum_{i = 1}^m \Pr\nolimits_{A \sim \cN(1/m)}[u \not \in C_{r' - 1} \text{ and } u \in \AlgDistorted(A \cup C_{r' - 1} \cup \{u\})]
		=
		p^{r'}_u
		\enspace,
	\end{align*}
	where the first equality holds since $C_{r'}$ can be obtained from $C_{r'- 1}$ by adding to the last set all the elements of $\cup_{i = 1}^m S_{r', i}$ that do not already belong to $C_{r' - 1}$, and the last equality holds since the distribution of $\cN_{r', i}$ conditioned on $u$ belonging to this set is equal to the distribution of $A \cup \{u\}$ when $A$ is distributed like $\cN(1/m)$.
	
	Since $C_1 \subseteq C_2 \subseteq \dotso \subseteq C_r$, the events $\Pr[u \in C_{r'} \setminus C_{r' - 1}]$ must be disjoint for different values of $r'$, which implies
	\[
	\sum_{r' = 1}^r p^{r'}_u
	=
	\sum_{r' = 1}^r \Pr[u \in C_{r'} \setminus C_{r' - 1}]
	=
	\Pr[u \in C_r] - \Pr[u \in C_0]
	=
	\Pr[u \in C_r]
	\enspace,
	\]
	where the last equality holds since $C_0 = \varnothing$ by definition.
\end{proof}

Using the last lemma, we can now prove lower bounds on the expected values of the sets $S_{r, i}$.
\begin{lemma} \label{lemma:bound-lovasz} 
	Let $\lov{g}$ and $\lov{\ell}$ be the Lov\'asz extensions of the functions $g$ and $\ell$, respectively. Then, for every two integers $1 \leq r \leq \eps^{-1}$ and $1 \leq i \leq m$,
	\begin{align*}
		\bE[f(S_{r, i})] \geq  (1 - e^{-1}) \cdot \lov{g}(\mathbf{1}_{S^*}-\mathbf{p}^r ) - \lov{\ell}(\mathbf{1}_{S^*}-\mathbf{p}^r ) \enspace,
	\end{align*}
	and
	\[
	\bE[f(S_{r, i})] \geq  (1 - e^{-1}) \cdot \lov{g}({\textstyle \sum_{r' = 1}^{r - 1} \mathbf{p}^{r'} }) - \lov{\ell}({\textstyle \sum_{r' = 1}^{r - 1}\mathbf{p}^{r'} })
	\enspace.
	\]
\end{lemma} 
\begin{proof}
	Let $R = \{u \in S^* \mid u \not \in \AlgDistorted(\cN_{r, i} \cup C_{r - 1} \cup \{u\})\}$, and let $O_{r, i}$ be some random subset of $S^*$ to be specified later which includes only elements of $\cN_{r, i} \cup C_{r - 1} \cup R$. By Lemma~\ref{lemma:lemma:dg-union},
	\begin{align*}
		S_{r, i}
		={} &
		\AlgDistorted(\cN_{r, i} \cup (\cup_{r' = 1}^{r - 1} \cup_{i' = 1}^m S_{r', i'}))\\
		={} &
		\AlgDistorted(\cN_{r, i} \cup C_{r - 1})
		=
		\AlgDistorted(\cN_{r, i} \cup C_{r - 1} \cup R)
		\enspace.
	\end{align*}
	Due to this equality and the fact that $|O_{r, i}| \leq |S^*| \leq k$, the guarantee of \AlgDistorted~\cite[Theorem~3]{harshaw2019submodular} implies:
	\[f(S_{r, i}) = g(S_{r, i}) - \ell(S_{r, i}) \geq (1 - e^{-1}) \cdot g(O_{r, i}) - \ell(O_{r, i}) \enspace.\]
	Therefore,
	\begin{align} \label{eq:general_O}
		\bE[f(S_{r, i})] &
		\geq \bE[   (1 - e^{-1}) \cdot g(O_{r, i}) - \ell(O_{r, i}) ] 
		= (1 - e^{-1}) \cdot \bE[  g(O_{r, i}) ] - \bE[ \ell(O_{r, i}) ] \\ \nonumber
		& \geq (1 - e^{-1}) \cdot \lov{g}(\bE[\mathbf{1}_{O_{r, i}}]) - \lov{\ell}(\bE[ \mathbf{1}_{O_{r, i}}]) \enspace,
	\end{align}
	where the second inequality holds since $\lov{g}$ is convex and $\lov{\ell}$ is linear (see the discussion before Observation~\ref{obs:basic_properties}).
	
	To prove the first part of the lemma, we now choose
	\[
	O_{r, i}= (C_{r - 1} \cap S^*) \cup R = (C_{r - 1} \cap S^*) \cup \left\{u \in S^{*}: u \notin \AlgDistorted(\cN_{r, i} \cup C_{r - 1} \cup\{u\})\right\}
	\enspace.
	\]
	One can verify that this choice obeys our assumptions about $O_{r, i}$; and moreover, since the distribution of $\cN_{r, i}$ is the same as that of $\cN(1/m)$, we get:
	\[
	\Pr\left[u \in O_{r, i}\right] =1-\operatorname{Pr}\left[u \notin O_{r, i}\right]=1-p^r_{u} \quad \forall\; u\in S^* \qquad \text{and} \qquad
	\mathbb{E}\left[\mathbf{1}_{O_{r, i}}\right] =\mathbf{1}_{S^*}-\mathbf{p}^r \enspace.
	\]
	The first part of the lemma now follows by combining the last equality with Inequality~\eqref{eq:general_O}.
	
	To prove the second part of this lemma, we choose $O_{r, i} = C_{r - 1} \cap S^*$. One can verify that this choice again obeys our assumptions about $O_{r, i}$; and moreover, by Lemma~\ref{lem:sum_p}, $\bE[\mathbf{1}_{O_{r, i}}] = \sum_{r' = 1}^{r - 1} \mathbf{p}^{r'}$. The second part of the lemma now follows by combining this equality with Inequality~\eqref{eq:general_O}.
\end{proof}

We are now ready to prove Theorem~\ref{thm:distributed}.
\begin{proof}[Proof of Theorem~\ref{thm:distributed}]
	Let $D$ be the output set of \cref{alg:distributed}. The definition of $D$ and Lemma~\ref{lemma:bound-lovasz}  together guarantee that for every $1 \leq r \leq \eps^{-1} - 1$ we have
	\[
	\bE[f(D)]  \geq \bE[f(S_{r, 1})] \geq  (1 - e^{-1}) \cdot \lov{g}(\mathbf{1}_{S^*}-\mathbf{p}^r ) - \lov{\ell}(\mathbf{1}_{S^*}-\mathbf{p}^r )
	\enspace,
	\]
	and additionally, 
	\[
	\bE[f(D)]  \geq \bE[f(S_{1/\eps, 1})] \geq  (1 - e^{-1}) \cdot \lov{g}({\textstyle \sum_{r = 1}^{1/\eps - 1} \mathbf{p}^r }) - \lov{\ell}({\textstyle \sum_{r = 1}^{1/\eps - 1} \mathbf{p}^r})
	\enspace.
	\]
	Therefore,
	\begin{align*}
		\bE[f(D)]
		\geq{} &
		\eps \cdot \mspace{-3mu} \sum_{r = 1}^{1/\eps - 1} [(1 - e^{-1}) \cdot \lov{g}(\mathbf{1}_{S^*}-\mathbf{p}^r ) - \lov{\ell}(\mathbf{1}_{S^*}-\mathbf{p}^r )] + \eps[(1 - e^{-1}) \cdot \lov{g}({\textstyle \sum_{r = 1}^{1/\eps - 1} \mathbf{p}^r }) - \lov{\ell}({\textstyle \sum_{r = 1}^{1/\eps - 1} \mathbf{p}^r})]\\
		\geq{} &
		(1 - e^{-1}) \cdot \lov{g}\left(\eps \cdot \sum_{r = 1}^{1/\eps - 1}(\mathbf{1}_{S^*}-\mathbf{p}^r) + \eps \cdot \sum_{r = 1}^{1/\eps - 1} \mathbf{p}^r\right) - \lov{\ell}\left(\eps \cdot \sum_{r = 1}^{1/\eps - 1}(\mathbf{1}_{S^*}-\mathbf{p}^r) + \eps \cdot \sum_{r = 1}^{1/\eps - 1} \mathbf{p}^r\right)\\
		={} &
		(1 - e^{-1}) \cdot \lov{g}\left((1 - \eps) \cdot \mathbf{1}_{S^*}\right) - \lov{\ell}\left((1 - \eps) \cdot \mathbf{1}_{S^*}\right)
		\geq
		(1 - \eps) \cdot \left[(1 - e^{-1}) \cdot g(S^*) -\ell(S^*)\right]
		\enspace,
	\end{align*}
	where the second inequality holds since $\lov{\ell}$ is linear and $\lov{g}$ is convex, and the last inequality follows again from the linearity of $\lov{\ell}$ and Observation~\ref{obs:basic_properties} because $f(\varnothing) = g(\varnothing) - \ell(\varnothing) = g(\varnothing) \geq 0$.
\end{proof}

	\section{Mode Finding of SLC Distributions} \label{sec:mode_finding}
In \cref{sec:intro}, we discussed the power of sampling from discrete probabilistic models (specifically SLC distributions), which encourages negative correlation, for data summarization. In this regard, recently, \citet{robinson2019flexible}  established a notion of $\gamma$-additively weak submodularity for SLC functions. By using this newly defined property, we guarantee the performance of our proposed algorithms for mode finding of a class of distributions that are derived from SLC functions. Following is the definition of $\gamma$-additively weak submodular functions.

\begin{definition}[Definition 1, \cite{robinson2019flexible}]
A set function $\rho\colon 2^{\cN} \rightarrow \bR$ is $\gamma$-additively weak submodular if for any $S \subseteq \cN$ and
$u, v \in  \cN \setminus S$ with $u \neq v$, we have
\begin{align*}
	 \rho(S)+\rho(S \cup\{u, v\}) \leq \gamma+\rho(S \cup \{u\})+\rho(S \cup \{v\}) \enspace.
\end{align*}
\end{definition}

In order to maximize a $\gamma$-additively weak submodular function $\rho$, we show that, with a little modification, $\rho$ can be converted to a submodular function $\Lambda$. We then show that $\Lambda$ in its own turn can be rewritten as the difference between a non-negative monotone submodular function and a modular function. 
 Towards this goal, we need to find a submodular function $\Lambda$ that is close to $\rho$, which is done in Lemma~\ref{lem:lamda}. Then, we explain in Lemma~\ref{lem:components} how to present $\Lambda$ as the difference between a monotone submodular function and a linear function, which allows us to optimize $\Lambda$ using the results of Theorem~\ref{thm:result} and Theorem~\ref{thm:distributed}. Finally, we show that even though we use these algorithms to optimize $\Lambda$, the solution they provide has a good guarantee with respect to the original $\gamma$-additively weak submodular function $\rho$.
With this new formulation, we improve the theoretical guarantees of  \citet{robinson2019flexible} in the offline setting, and provide our streaming and  distributed  solutions for the mode finding problem under a cardinality constraint $k$.
Specifically, by using either our proposed streaming or distributed algorithms (depending on whether the setting is a streaming or a distributed setting), we can get a scalable solution with a guarantee with respect to $\rho$, and in particular, a guarantee for the task of finding the mode of an SLC distribution.
We should point out that it is also possible to use the distorted greedy algorithm \textup{\cite[Algorithm~1]{harshaw2019submodular}} to optimize $\Lambda$ in the offline setting.

\newcommand{\LemmaLamda}{For a $\gamma$-additively weak submodular function $\rho$,	the function $\Lambda(S) \triangleq \rho(S) - \frac{\gamma}{2} \cdot |S|\cdot (|S| - 1)$ is submodular.}

\begin{lemma} \label{lem:lamda}
	\LemmaLamda
\end{lemma}
\begin{proof}
	For every set $S$ and two distinct elements $u, v \not \in S$, the $\gamma$-additively weak submodularity of $\rho$ implies
	\[
	\rho(S)+ \rho(S \cup\{u, v\})  \leq \gamma+\rho(S \cup \{u\})+\rho(S \cup \{v\})
	\enspace.
	\]
	Rearranging this inequality now gives
	\begin{align*}
	&
	\rho(S) - \frac{\gamma \cdot   |S|  \cdot (|S| - 1))}{2}  +  \rho(S \cup\{u, v\})   - \frac{\gamma \cdot   (|S| + 2) \cdot (|S| + 1)}{2}   \\
	& \hspace{90pt}  \leq   \rho(S \cup \{u\})+ \rho(S \cup \{v\})  -  \frac{2 \cdot \gamma \cdot  (|S|+1) \cdot |S| }{2}
	\enspace,
	\end{align*}
	which, by the definition of $\Lambda$, is equivalent to
	\[
	\Lambda(S)+\Lambda(S \cup\{u, v\})  \leq \Lambda(S \cup \{u\})+\Lambda(S \cup \{v\}) \enspace.
	\qedhere
	\]
\end{proof}

Now, let us define the modular function $\ell(S) = \sum_{u \in S} \ell_u$, where $\ell_u \triangleq \max\{ \Lambda(\cN \setminus u) - \Lambda(\cN), 0 \} =\max\{ \rho(\cN \setminus u) - \rho(\cN)  + \gamma \cdot (|\cN| - 1), 0 \}  $.

\newcommand{\LemmaComponents}{The	function $g(S) \triangleq \Lambda(S) + \ell(S)$ is monotone and submodular. Furthermore, if $\rho(\emptyset) \geq 0$, then $g(S)$ is also non-negative because $\ell(\varnothing) = 0$.}
\begin{lemma} \label{lem:components}
	\LemmaComponents
\end{lemma}
\begin{proof}
	To see that $g(S)$ is submodular, recall that $\Lambda(S)$ is submodular and that the summation of a submodular function with a modular function is still submodular. To prove the monotonicity of $g(S)$, we show that for all sets $S \subseteq \cN$ and elements $u \in \cN \setminus S$: $g(u \mid S) \geq 0$.
	\begin{align*}
	g(u \mid S) = \Lambda(u \mid S) + \ell(u \mid S) & = \Lambda(u \mid S) + \ell_u = \Lambda(u\mid S) +  \max\{ \Lambda(\cN \setminus u) - \Lambda(\cN), 0 \} \\
	&	\geq \Lambda(u\mid S) + \Lambda(\cN \setminus u) - \Lambda(\cN) = \Lambda(u \mid S) - \Lambda(u \mid \cN \setminus u) \geq 0 \enspace,
	\end{align*}
	where the last inequality follows from the submodularity of $\Lambda$.
\end{proof}

We now show that by optimizing $\Lambda$ under a cardinality constraint $k$, by using either our proposed streaming or distributed algorithms (depending on whether the setting is a streaming or a distributed setting), we can get a scalable solution with a guarantee with respect to $\rho$, and in particular, a guarantee for the task of finding the mode of an SLC distribution.

\begin{corollary} \label{cor:mode-offline}
	Assume $\rho\colon 2^{\cN} \rightarrow \bR$ is a $\gamma$-additively weak submodular function. Then, when given $\Lambda$ as the objective function, \AlgDistributed (Algorithm~\ref{alg:distributed}) returns a solution $R$ such that
\begin{align*}
	\bE[\rho(R)]  \geq  (1 - \eps) \cdot \left[ (1 - e^{-1}) \cdot \rho(\Opt) - e^{-1} \cdot \ell(\Opt) \right]    - \frac{\gamma \cdot [(1 -e^{-1}) \cdot l\cdot (l -1 ) - \bE[|R| \cdot (|R| - 1) ]]}{2}  \enspace, 
\end{align*}
where $\Opt \in \argmax_{|S| \leq k}  \rho(S)$ and $l = |\Opt| \leq  k$.
\end{corollary}

\begin{proof}
	Using the guarantee of Theorem~\ref{thm:distributed} for the performance of  \AlgDistributed for maximizing the function $\Lambda(S) = g(S) - \ell(S)$ in the distributed setting under a cardinality constraint $k$, we get
	\[
	\bE[g(R) - \ell(R)] \geq (1 - \eps) \cdot \left[ (1 - e^{-1}) \cdot g(\Opt) - \ell(\Opt) \right]
	\enspace,
	\]
	which implies, by the definition of $g$,
	\[
	\frac{\bE[\Lambda(R)]}{1 - \eps}
	\geq
	(1 - e^{-1}) \cdot (\Lambda(\Opt) + \ell(\Opt) )-  \ell(\Opt)
	=
	(1 - e^{-1}) \cdot \Lambda(\Opt) - e^{-1} \cdot \ell(\Opt)
	\enspace.
	\]
	Using the definition of $\Lambda$ now, we finally get
	\begin{align*}
	\bE[\rho(R)] & \geq (1 - \eps) \cdot \left[ (1 - e^{-1}) \cdot \rho(\Opt) - e^{-1} \cdot \ell(\Opt) \right] \\ 
	&  \qquad \qquad - \frac{\gamma \cdot [ (1 - \eps) \cdot (1 -e^{-1}) \cdot |\Opt| \cdot (|\Opt| -1 ) - \bE[|R| \cdot (|R| - 1)]]}{2} \enspace,
	\end{align*}
	which implies the corollary since $(1 -e^{-1}) \cdot |\Opt| \cdot (|\Opt| -1 )$ is non-negative.
\end{proof}

The following corollary shows the guarantee obtained by \AlgThreshold as a function of the input parameter $r$. When the best choice for $r$ is unknown, \AlgStream roughly obtains this guarantee for the best value of $r$, as discussed in \cref{sec:streaming}.

\begin{corollary} \label{cor:mode-streaming}
	Assume $\rho\colon 2^{\cN} \rightarrow \bR$ is a $\gamma$-additively weak submodular function. Then, when given $\Lambda$ as the objective function, \AlgThreshold (Algorithm~\ref{alg:threshold-first}) returns a solution $R$ such that $\rho(R)$ is at least
\begin{align*}
	  (h(r) - \epsilon)  \cdot \rho(\Opt) - ( \alpha(r) - r -1 + \epsilon)   
	 \cdot \ell(\Opt) 
	- \frac{\gamma \cdot [ (h(r) - \epsilon) \cdot l \cdot (l -1 ) - |R| \cdot (|R| - 1) ]}{2} \enspace,
\end{align*}
	where $\Opt \in \argmax_{|S| \leq k}  \rho(S)$ and $l = |\Opt| \leq  k$. 
\end{corollary}

\begin{proof}
	By Theorem~\ref{thm:result},
	\[
	g(R) - \ell(R) \geq (h(r) - \epsilon) \cdot g(\Opt) - r \cdot \ell(\Opt)
	\enspace,
	\]
	which implies, by the definition of $g$,
	\[
	\Lambda(R) \geq (h(r) - \epsilon) \cdot \left(\Lambda(\Opt) + \ell(\Opt) \right) - r \cdot \ell(\Opt)
	= (h(r) - \epsilon)  \cdot \Lambda(\Opt) - ( r - h(r) + \epsilon) \cdot \ell(\Opt)
	\enspace.
	\]
	Using the definition of $\Lambda$ now, we finally get
	\begin{align*}
	\rho(R) & \geq  (h(r) - \epsilon)  \cdot \rho(\Opt) - ( r - h(r) + \epsilon)  \cdot \ell(\Opt) \\ 
	&  \qquad \qquad - \frac{\gamma \cdot [ (h(r) - \epsilon) \cdot |\Opt| \cdot (|\Opt| -1 ) - |R| \cdot (|R| - 1) ]}{2} \enspace.
	\end{align*}
	The corollary now follows by observing that $r - h(r) = \frac{\sqrt{4r^2 + 1} - 1}{2}$.
\end{proof}

Note that in \cite[Theorem~12]{robinson2019flexible} and Corollary~\ref{cor:mode-offline},  if the value of the linear function is considerably larger than the values of functions $\eta$ or $\rho$, then the parts that depend on the optimal solution $\Opt$ in the right-hand side of these results could be negative, which makes the bounds trivial.
The main explanation for this phenomenon is that the distorted greedy algorithm does not take into account the relative importance of $g$ and $\ell$ to the value of the optimal solution. 
On the other hand, the \textbf{distinguishing feature} of our streaming algorithm is that, by guessing the value of $\beta_{\Opt}$, it can find the best possible scheme for assigning weights to the importance of the submodular and modular terms.
Therefore, \AlgStream, even in the scenarios where the linear cost is large, can find solutions with a non-trivial provable guarantee. 
In the experiments presented at \cref{sec:mode_finding_experiment}, we showcase two facts that: (i) \AlgStream could be used for mode finding of strongly log-concave distributions with a \text{provable guarantee}, and (ii)  choosing an accurate estimation of $\beta_{S^*}$ plays an important role in this optimization procedure.

	\section{Experiments}\label{sec:results}
In this section we present the experimental studies we have performed to show the applicability of our approach.
In the first set of experiments (\cref{section:vertex-cover}), we compare the performance of our proposed streaming algorithm with that of \AlgDistorted \cite{harshaw2019submodular}, vanilla greedy and sieve-streaming \cite{badanidiyuru2014streaming}. 
The main message of these experiments is to show that our proposed distorted-streaming algorithm outperforms both vanilla greedy and sieve streaming, and performs comparably with respect to the distorted-greedy algorithm in terms of the objective value---despite the fact that our proposed algorithm makes only a single pass over the data, while distorted-greedy has to make $k$ passes (which can be as large as $\Theta(n)$ in the worst case).
In the second set of experiments (presented at \cref{sec:mode_finding_experiment}), we evaluate the performance of \AlgStream on the task of finding the mode of SLC distributions.
In the third set of experiments (\cref{sec:distributed-experiment}), we compare \AlgDistributed with distributed greedy.
In the final experiments (\cref{sec:data_summarization}), we demonstrate the power of our proposed regularized model by comparing it with the alternative approach of maximizing a submodular function subject to cardinality and single knapsack constraints. 
In the latter case, the goal of the knapsack constraint is to limit the linear function $\ell$ to a pre-specified budget while the algorithm tries to maximize the monotone submodular function $g$.

\subsection{How Effective is \AlgStream?}   \label{sec:alg-experiment}

\label{section:vertex-cover}

In this experiment, we compare \AlgStream with distorted-greedy, greedy and sieve-streaming in the setting studied in \cite[Section~5.2]{harshaw2019submodular}.
In this setting, there is a submodular function $f$ over the vertices of a directed graph $G = (V, E)$. To define this function, we first need to have a weight function $w\colon V \rightarrow \bR_{\geq 0}$ on the vertices. 
For a given vertex set $S \subseteq V$, let $N(S)$ denote the set of vertices which are pointed to by $S$, i.e., $N(S) \triangleq \{v \in V \mid \exists u \in S \text{ such that } (u, v) \in E\}$. Then, we have
$f(S) \triangleq g(S)-\ell(S) = \sum_{u \in N(S) \cup S} w_{u}-\sum_{u \in S} \ell_u.$
Following~\citet{harshaw2019submodular}, we assigned a weight of $1$ to all nodes and set $\ell_u = 1 + \max\{0, d_u - q\}$, where $d_u$ is the out-degree of node $u$ in the graph $G(V,E)$ and $q$ is a parameter.
In our experiment, we used real-world graphs from \cite{snapnets}, set $q = 6$, and ran the algorithms for varying cardinality constraint $k$.
In \cref{fig:vertex-cover}, we observe that for all four networks, distorted greedy, which is an offline algorithm, achieves the highest objective values.
Furthermore, we observe that \AlgStream consistently outperforms both greedy and sieve-streaming, which demonstrates the effectiveness of our proposed method.

\begin{figure*}[htb!] 
	\centering  
	\subfloat[Social graph]{\includegraphics[height=31.7mm]{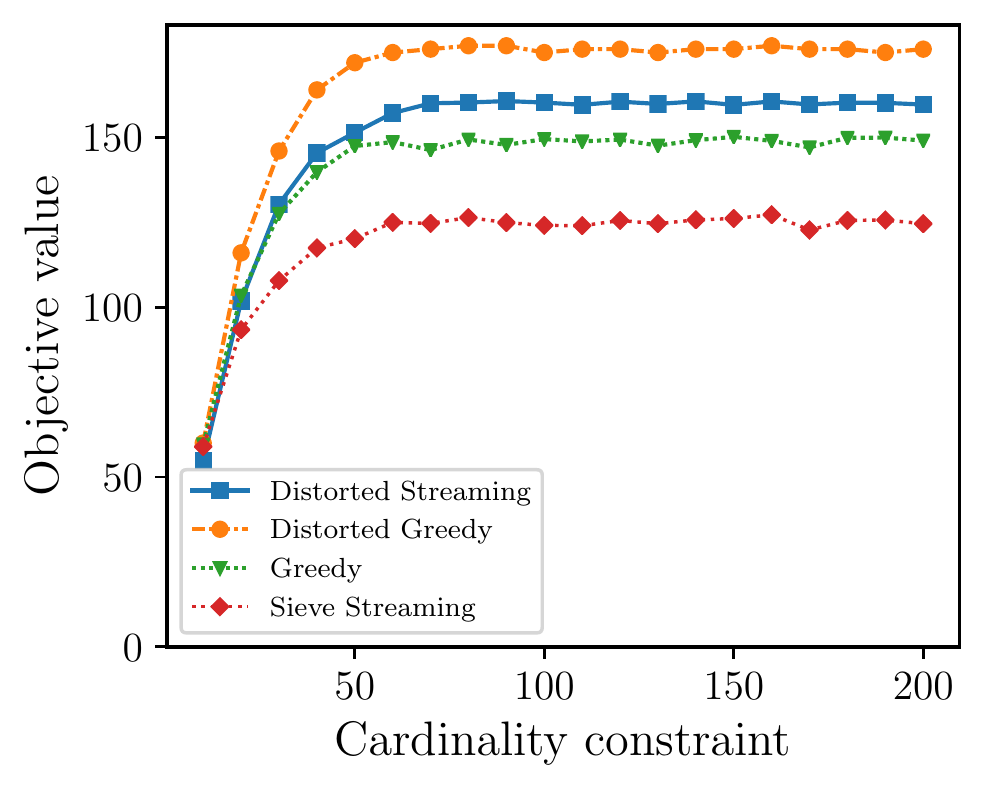}\label{fig:facebook}}
	\subfloat[Facebook ego network]{\includegraphics[height=31.7mm]{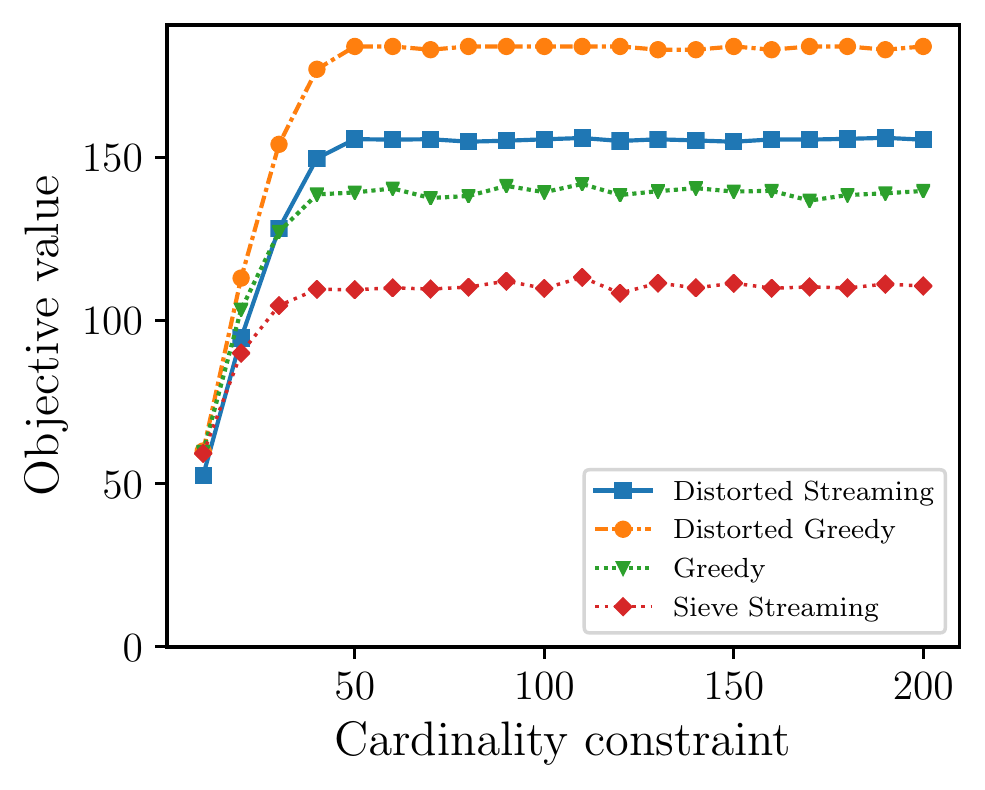}\label{fig:facebook_ego}}
	\subfloat[EU Email]{\includegraphics[height=31.7mm]{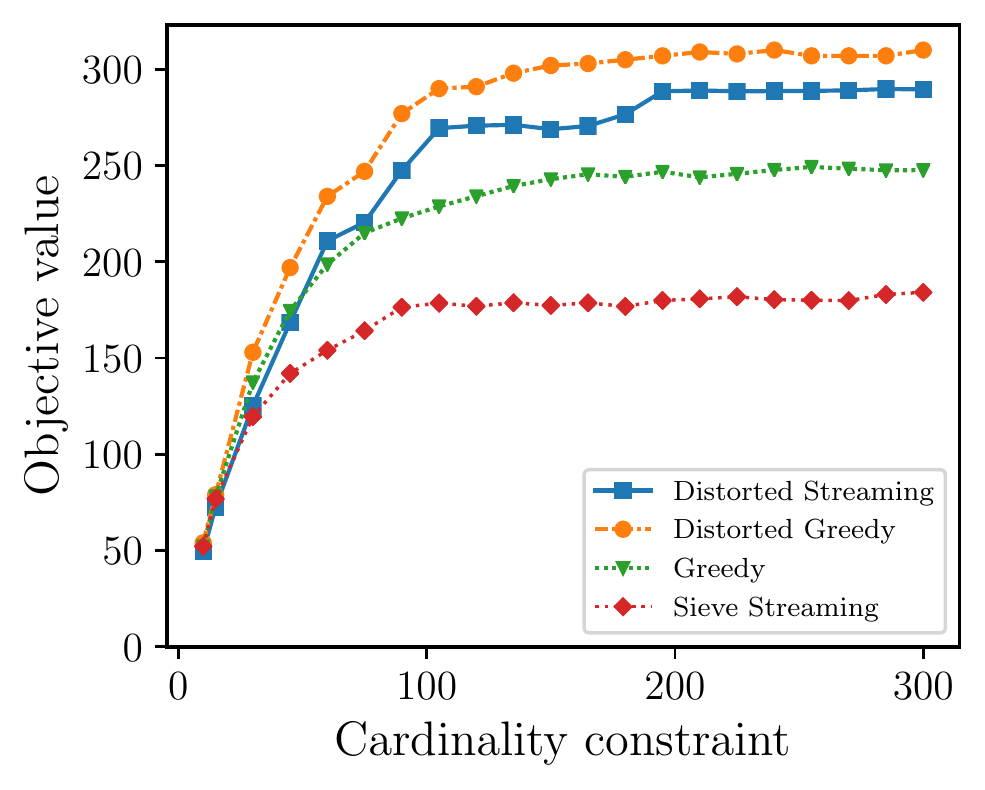}\label{fig:email}}
	\subfloat[Wikipedia vote]{\includegraphics[height=31.7mm]{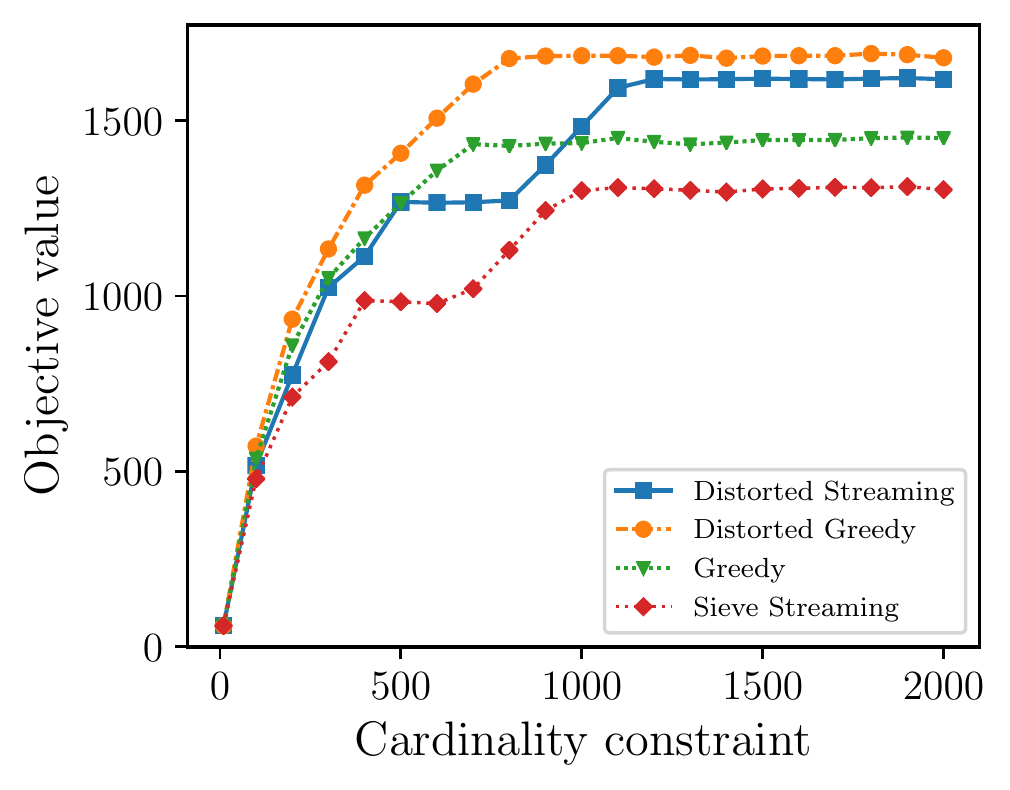}\label{fig:wikipedia}}
	\caption{Directed vertex cover: we compare objective values by varying the cardinality constraint $k$. }
	\label{fig:vertex-cover}
\end{figure*}

\subsection{Mode Finding of SLC Distributions: Experimental Evaluations} \label{sec:mode_finding_experiment}

In this section, we compare the performance of \AlgStream with the performance of distorted greedy, vanilla greedy and sieve streaming on the problem of mode finding for an SLR distribution.
We consider a distribution $\nu(S) \propto \sqrt{\det \left(L_{S}\right)} \cdot \mathbf{1} \{|S| \leq d\}$, where $L$ is an $n \times n$ PSD matrix.  
Here, $L_S$ corresponds to the submatrix of $L$, where the rows and columns are indexed by elements of $S$ \cite{robinson2019flexible}. 
In the optimization procedure, our goal is to maximize $\rho(S) \triangleq \log(\nu(S))$.
To generate the random matrix $L$, we first sample a diagonal matrix $D$ and a random PSD matrix $Q$, and then assign $L = QDQ^{-1}$.
Each diagonal element of $D$  is from  a log-normal distribution with a probability mass function $p(x) = \frac{1}{\sigma x \sqrt{2 \pi}} \exp(-\frac{(\ln(x) - \mu)^2}{2 \sigma^2})$,
where $\mu$  and $\sigma$ are the mean and standard deviation of the normally distributed logarithm of the variable, respectively. This log-normal distribution allows us to have a PSD matrix where the eigenvalues have a heavy-tailed distribution.
In these experiments, we set $n=1000, d = 100, \mu = 1.0$ and $\sigma = 1.0$.

In \cref{fig:mode-finding}, we observe that the outcome of \AlgStream outperforms sieve streaming. This is mainly a result of the fact that \AlgStream estimates the value of $\beta_{\Opt}$ and uses the best possible value for $r$.
Furthermore, we see that vanilla greedy performs better than distorted greedy, and for cardinality constraints larger than $k=20$, the performance of distorted greedy degrades. 
This observation could be explained by the fact that the linear cost for each element $u$ is comparable to the value of $g(u)$ (or marginal gain of $u$ to any set $S$). Therefore, distorted greedy does not pick any element in the first few iterations when $k$ is large enough, i.e., when $(1 - \frac{1}{k})^{k-(i+1)}$ is small.
It is worth mentioning that while the performance of the greedy algorithm is the best for this specific application, only \AlgStream and distorted greedy have a theoretical guarantee.

\begin{figure}[ht] 
	\centering
	\includegraphics[height=50mm]{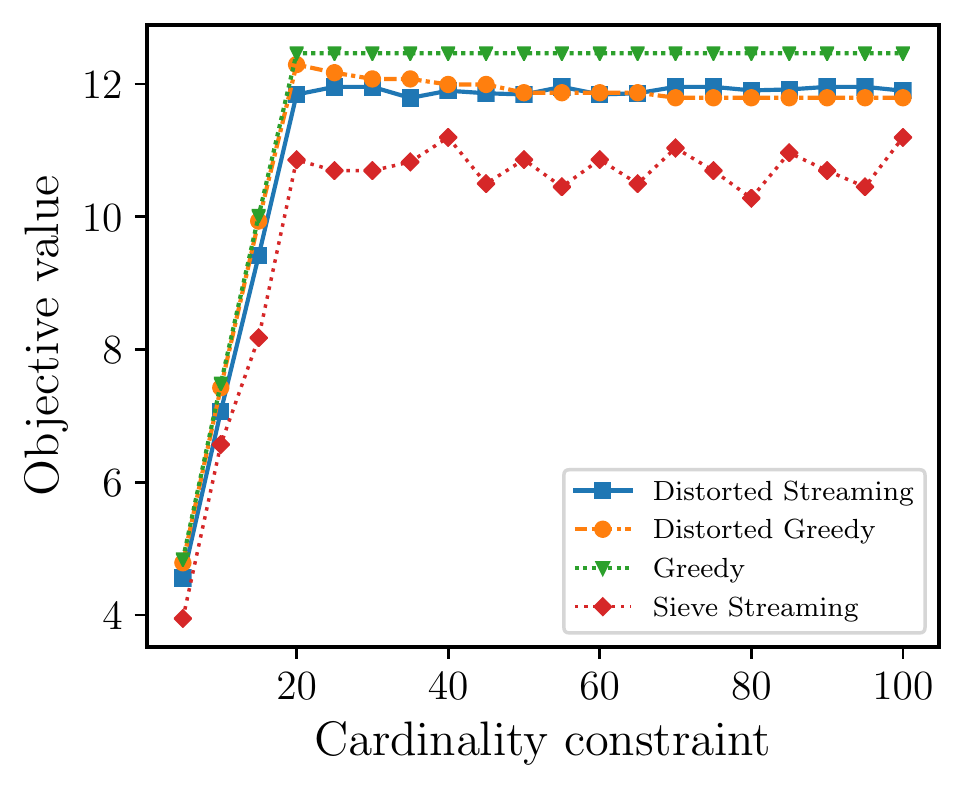}
	\caption{We want to find the mode of a distribution $\nu(S) \propto \sqrt{\det \left(L_{S}\right)} \cdot \mathbf{1} \{|S| \leq d\}$ for a PSD matrix $L$. 
		For the objective value, we report $\log(\nu(S))$.\label{fig:mode-finding}}
\end{figure}

\subsection{Distributed Setting} \label{sec:distributed-experiment}

In this section, we compare \AlgDistributed with the distributed greedy algorithm of~\citet{barbosa2016new}.
We evaluate the performance of these algorithms over several large graphs \cite{snapnets} in the setting of \cref{section:vertex-cover} under a cardinality constraint $k=\numprint{1000}$, where we set $q = 50$.
For both algorithms, we set the number of computational rounds to 10.
The first graph is the Amazon product co-purchasing network with	$n=\numprint{334863}$ vertices;
the second one is the DBLP collaboration network with $n=\numprint{317080}$ vertices;
the third graph is Youtube with $n=\numprint{1134890}$ vertices; and
the last graph we consider is the Pokec social network, the most popular online social network in Slovakia, with $n=\numprint{1632803}$ vertices.
For each graph, we set the number of machines to $ m = \lceil \nicefrac{n}{4000} \rceil$.
From \cref{fig:distributed}, we can see that the objective values of  $\AlgDistributed$ exceed the results of distributed greedy for all four graphs.
\begin{figure}[htb!] 
	\centering
	\includegraphics[height=50mm]{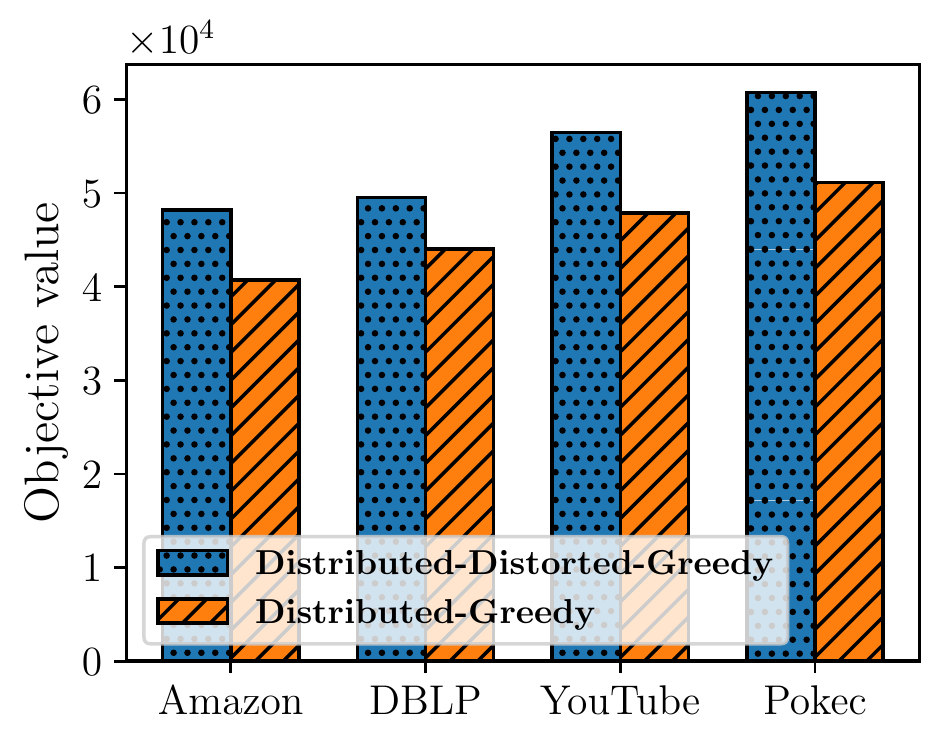}
	\caption{We compare the \AlgDistributed with the distributed greedy algorithm under a cardinality constraint with $k = \numprint{1000}$. The number of computational rounds is set to 10. \label{fig:distributed}}
\end{figure}

\subsection{Regularized Data Summarization} \label{sec:data_summarization}
In this section, through an extensive set of experiments, we answer the following two questions:
\begin{itemize}
\item How does \AlgStream perform with respect to sieve-streaming and distorted greedy on real-world data summarization tasks?
\item Is our proposed modeling approach (maximizing diversity while considering costs of items  as a regularization term  in a single function) favorable to methods which try to maximize a submodular function subject to a knapsack constraint?
\end{itemize}
We consider three state-of-the-art algorithms for solving the problem of submodular maximization with a cardinality and a knapsack constraint:
FANTOM \cite{mirzasoleiman2016fast}, Fast \cite{badanidiyuru2014fast} and  Vanilla Greedy Dynamic Program \cite{mizrachi2019tight}.
For the sake of fairness of our experiments, we used these three algorithms to maximize the submodular function $g$ under 50 different knapsack capacities $c$ in the interval $0.1 \leq c \leq 100$ and reported the solution maximizing $g(S) - \ell(S)$. 
We note that, for the computational complexities of these algorithms, we report the number of oracle calls used by a single one out of their 50 runs (one for each different knapsack capacity), which gives these offline algorithms a considerable edge in our comparisons.

\subsubsection{Online Video Summarization} \label{sec:video}
In this task, we consider the online video summarization application, where a stream of video frames comes, and one needs to provide a set of $k$ representative frames as the summary of the whole video.
In this application, the objective is to select a subset of frames in order to maximize a utility function $g(S)$ (which represents the diversity), while minimizing the total entropy of the selection.
We use a non-negative modular function $\ell(S)$ to represent the entropy of the set $S$, which could be interpreted as a proxy of the storage size of $S$.

We used the pre-trained ResNet-18 model \cite{he2016deep} to extract features from frames of each video. 
Then, given a set of frames, we defined a matrix $M$ such that $M_{ij}= e^{-\textrm{dist}(x_i,x_j)}$, where $\text{dist}(x_i,x_j)$ denotes the distance between the feature vectors of the $i$-th and $j$-th frames, respectively.
One can think of $M$ as a similarity matrix among different frames of a video.
The utility of a set $S \subseteq V$ is defined as a non-negative and monotone submodular objective $g(S) = \log \det (\mathbf{I} + \alpha M_S)$, where $\mathbf{I}$ is the identity matrix, $\alpha$ is a positive scalar and $M_S$ is the principal sub-matrix of the similarity matrix $M$ indexed by $S$. Informally, this function is meant to measure the diversity of the vectors in $S$.
To sum-up, we want to maximize the following function under a cardinality constraint $k$:
$ 
f(S) \triangleq g(S) - \ell(S) = \log \det( \mathbf{I} + \alpha M_S ) - \sum_{u \in S} \mathrm{H}_u,
$
where $\mathrm{H}_u$ represents the entropy of frame $u$.

\begin{figure*}[htb!] 
	\centering  
	\subfloat[Video number 13] {\includegraphics[height=32mm]{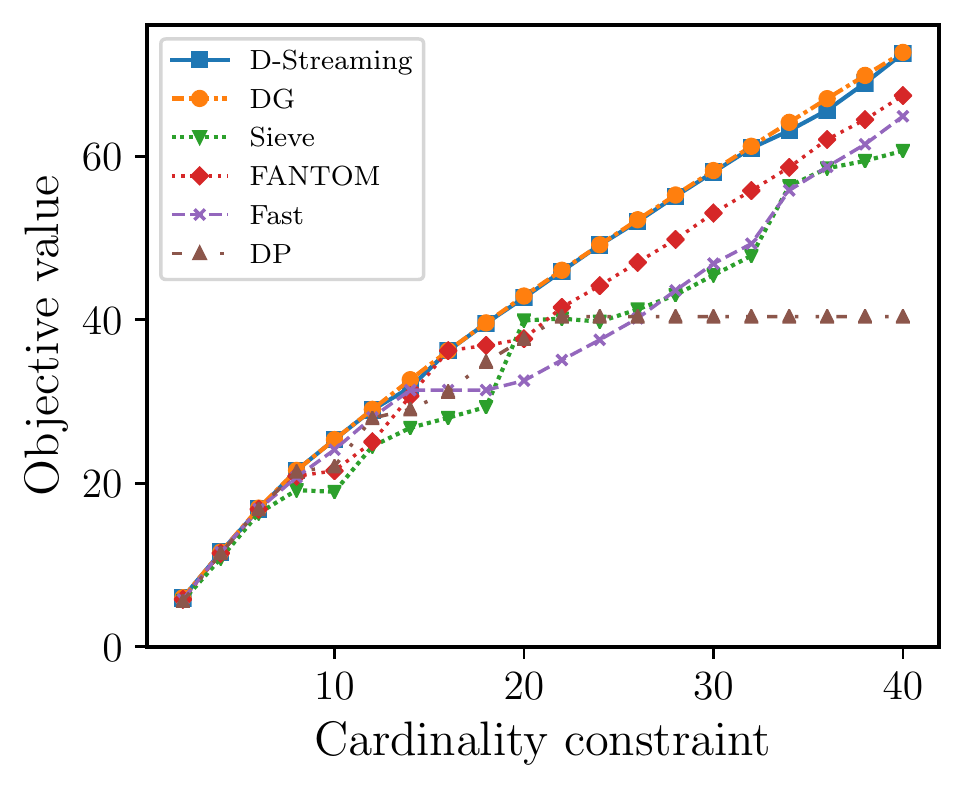}\label{fig:image-f-13}}
	\subfloat[Video number 15] {\includegraphics[height=32mm]{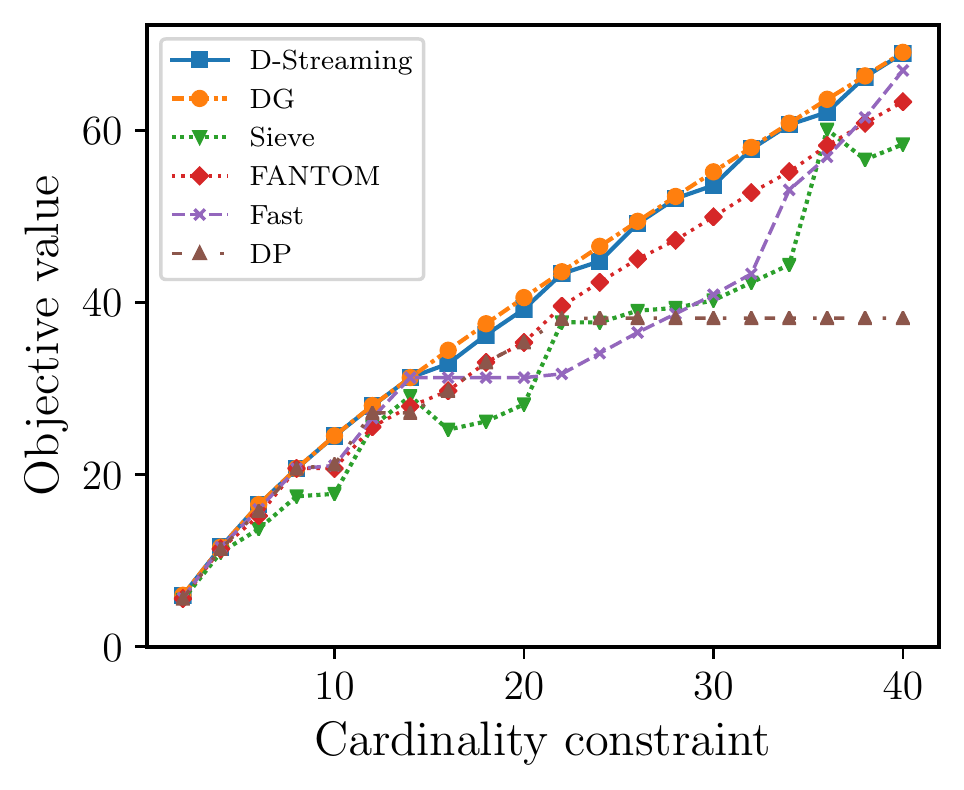}\label{fig:image-f-15}}
	\subfloat[Video number 13] {\includegraphics[height=32mm]{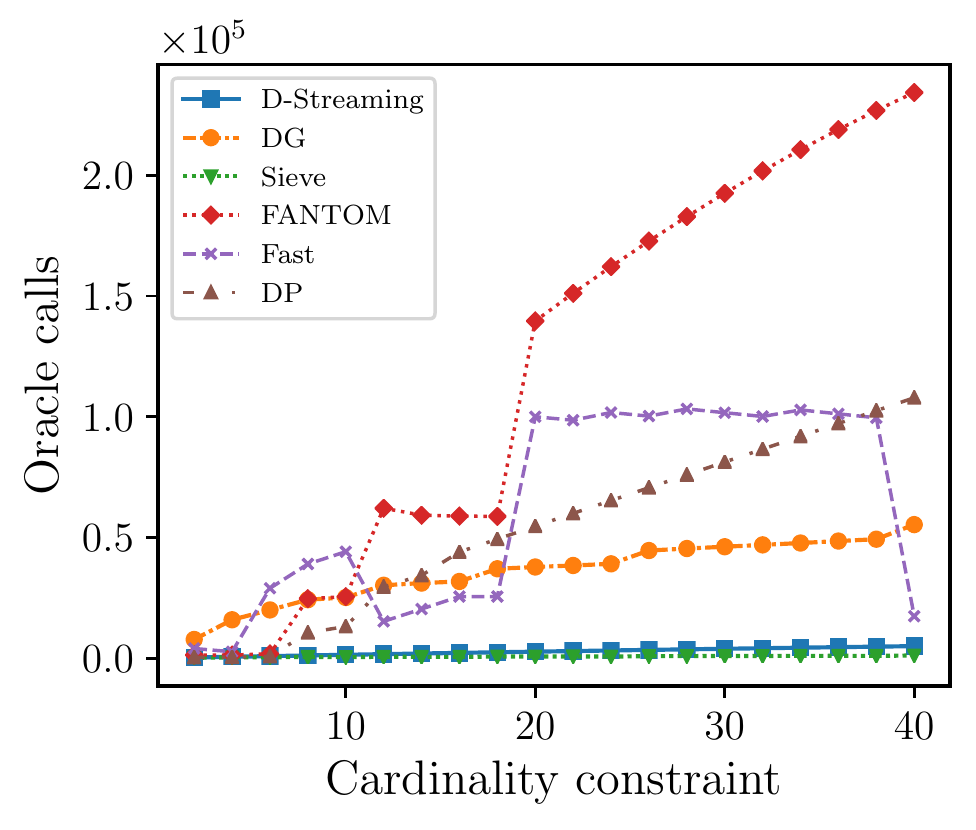}\label{fig:image-o-13}}
	\subfloat[Video number 15] {\includegraphics[height=32mm]{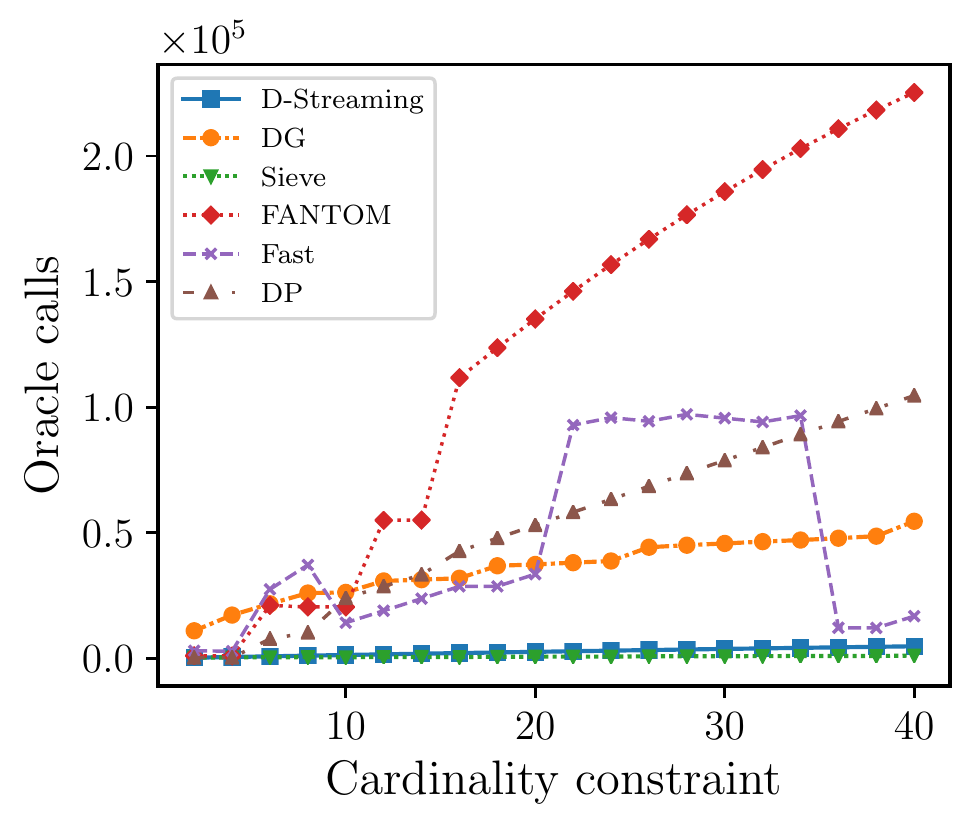}\label{fig:image-o-15}}
	\qquad
	\subfloat[The summaries produced by \AlgStream for video number 14.] {\includegraphics[height=45mm]{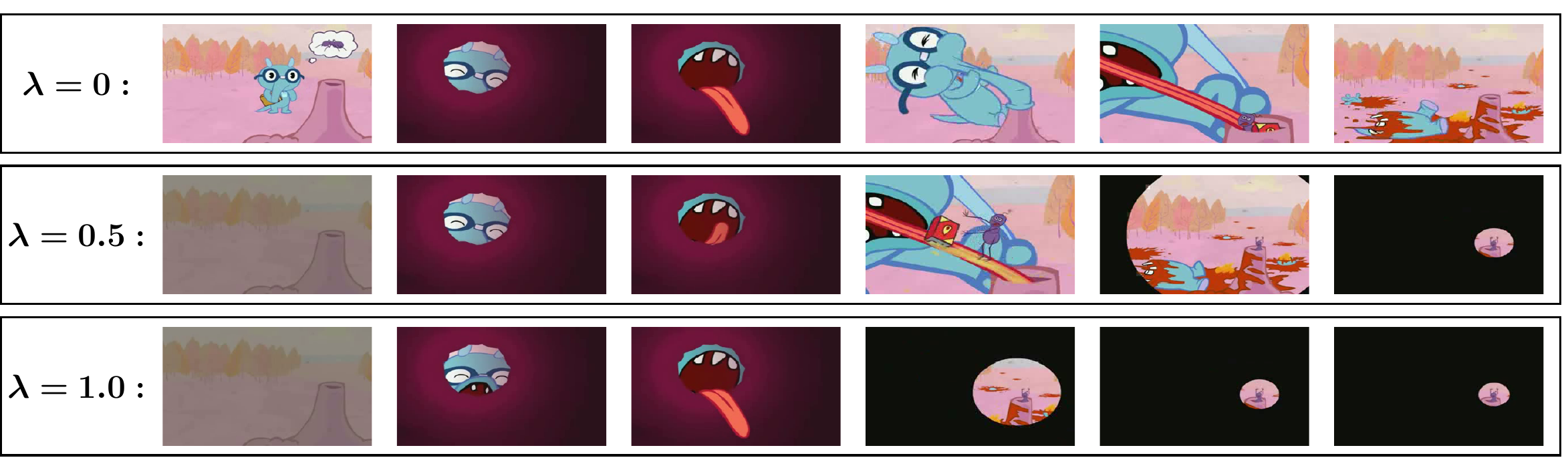}\label{fig:sampled-frames}}
	\caption{Movie frame summarization: 
		For each frame $u$ the linear cost $\ell_u$ is the entropy of that frame. (a) and (b) compare the objective values. (c) and (d) compare the computational complexities based on the number of oracle calls. In experiment (e), the input function is  $f(S) \triangleq g(S) - \lambda \cdot \ell(S)$ for $\lambda \in \{0,0.5, 1.0\}$, where we set the cardinality constraint to $k=6$.
		\label{fig:image}}
\end{figure*}

In the first experiment, we summarized the frames of videos 13 and 15 from the VSUMM dataset~\cite{de2011vsumm}\footnote{\url{https://sites.google.com/site/vsummsite/}}, and compare the above mentioned algorithms based on their objective values and number of oracle values for varying cardinality constraint $k$.
From \cref{fig:image-f-13,fig:image-f-15} we conclude that (i) the quality of the solutions returned by \AlgStream is as good as the quality of the results of distorted greedy, 
(ii) distorted greedy clearly outperforms sieve-streaming, and
(iii) the objective values of \AlgStream and distorted greedy are both larger than the corresponding values produced by Greedy Dynamic Program, Fast and FANTOM.
This confirms that directly maximizing the function $f$ provides higher utilities versus maximizing the function $g$ and setting a knapsack constraint over the modular function $\ell$.
In \cref{fig:image-o-13,fig:image-o-15}, we observe that the computational complexity of \AlgStream and sieve streaming is several orders of magnitudes better than the computation complexity of the other algorithms, which is consistent with their need to make only a single pass over the data.

Next, we study the effect of the linear cost function (in the other words, the importance we give to the entropy of frames) on the set of selected frames. 
For this reason, we run \AlgStream on the frames from video number 14. 
The objective function is $f(S) \triangleq g(S) - \lambda \cdot \ell(S)$ for $\lambda \in \{0,0.5, 1.0\}$. 
In this experiment, we set the cardinality constraint to $k = 6$.
In \cref{fig:sampled-frames}, we observe that by increasing $\lambda$ the
entropy of the selected frames decreases. This is evident from the fact that the color diversity of pixels in each frame reduces for larger values of $\lambda$.
Consequently, the  representativeness of the selected subset decreases. 
Indeed, while it is easy to understand the whole story of this animation from the output produced for $\lambda = 0$, some parts of the story are definitely missing if we set $\lambda$ to $1.0$.

\subsubsection{Yelp Location Summarization}

In this summarization task, we want to summarize a dataset of business locations. 
For this reason, we consider a subset of Yelp's businesses, reviews and user data \cite{yelporig}, referred to as the Yelp Academic dataset \cite{yelp}. 
This dataset contains information about local businesses across 11 metropolitan areas.
The features for each location are extracted from the description of that location and related user reviews.
The extracted features cover information regarding several attributes, including parking options, WiFi access, having vegan menus, delivery options,  possibility of outdoor seating,  being good for groups.\footnote{For the feature extraction, we used the script provided at \url{https://github.com/vc1492a/Yelp-Challenge-Dataset}.}

The goal is to choose a subset of  businesses locations, out of a ground set $\cN = \{1, \dots , n \} $,  which provides a good summary of all the existing locations. We calculated the similarity matrix $M \in \bR^{n \times n}$ between locations using the same method described in \cref{sec:video}.
For a selected set $S$, we assume each location $i \in \cN$ is represented by the location from the set $S$ with the highest similarity to $i$. This makes it natural to define the total utility provided by set $S$ using the set function
\begin{align} \label{eq:facility}
f(S) \triangleq g(S)-\ell(S) =\frac{1}{n} \sum_{i=1}^{n} \max_{j \in S} M_{i,j} - \sum_{u \in S} \ell_u \enspace.
\end{align}
Note that $g(S)$ is monotone and submodular \cite{krause12survey,frieze1974cost}.
For the linear function $\ell$ we consider two scenarios: 
(i) in the first one, the cost assigned to each location is defined as its distance to the downtown in the city of that location. 
ii) in the second scenario, the linear cost of each location $u$ is the distance between $u$ and the closest international airport in that area.
The intuitive explanation of the first linear function is that while we look for the most diverse subset of locations as our summary, we want those locations to be also close  enough to the down-town in order to make commute and access to other facilities easier. For the second linear function, we want the selected locations to be in the vicinity of airports.

From  \cref{eq:facility}  it is evident that computing the objective function requires access to the entire dataset $\cN$, which in the streaming setting is not possible.
Fortunately, however, this function is additively decomposable \cite{mirzasoleiman2013distributed} over the ground set $\cN$. 
Therefore, it is possible to estimate \cref{eq:facility} arbitrarily close to its exact value as long as we can sample uniformly at random from the data stream \cite[Proposition 6.1]{badanidiyuru2014streaming}.
In this section, to sample randomly from the data stream and to have an accurate estimate of the function, we use the reservoir sampling technique explained in \cite[Algortithm~4]{badanidiyuru2014streaming}.

In \cref{fig:yelp}, we compare algorithms for varying values of $k$ while we consider the two  different linear functions $\ell$. We observe that distorted greedy returns the solutions with the highest utilities. 
The performance of \AlgStream is comparable with that of the offline algorithms, and it clearly surpasses sieve-streaming. 
In addition, our experiments demonstrate that \AlgStream (and similarly sieve-streaming) requires orders of magnitude fewer oracle evaluations.

\begin{figure*}[ht] 
	\centering  
	\subfloat[$\ell=$ distance to down-town] {\includegraphics[height=32mm]{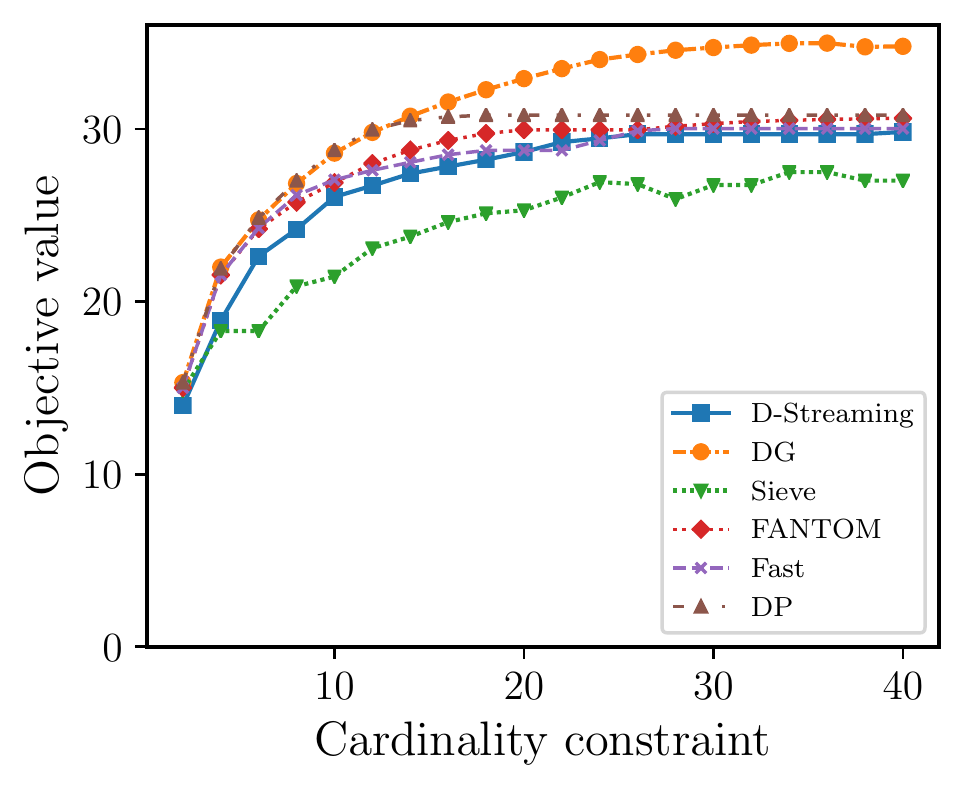}\label{fig:yelp-f-0}}
	\subfloat[$\ell=$ distance to airport] {\includegraphics[height=32mm]{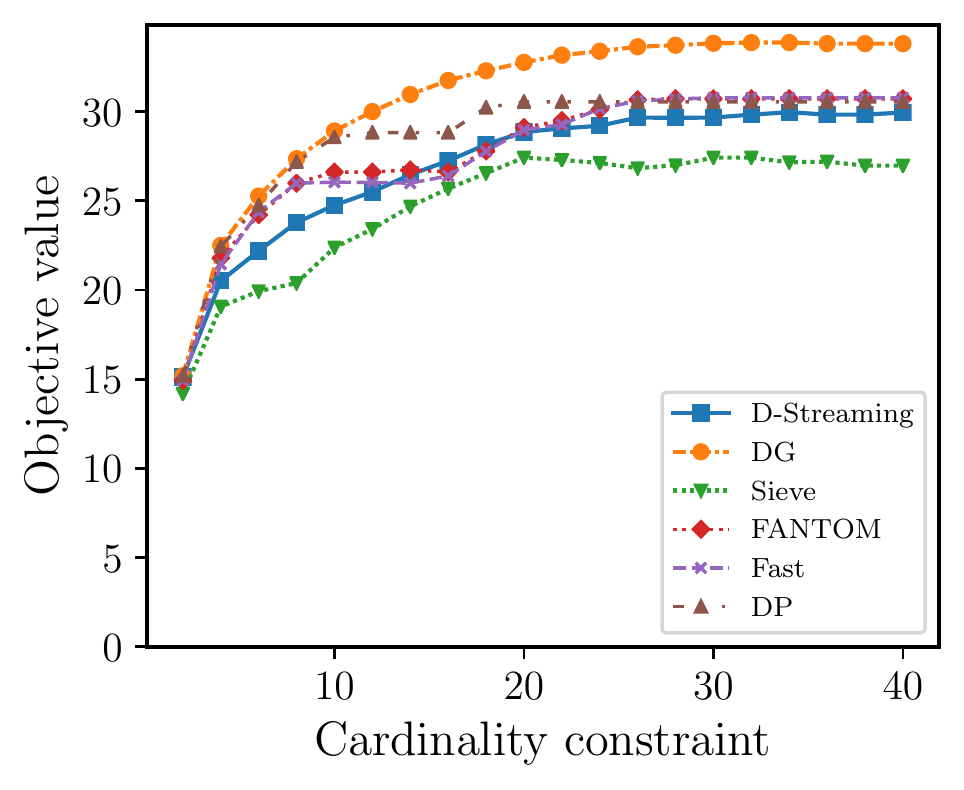}\label{fig:yelp-f-1}}
	\subfloat[$\ell=$ distance to down-town] {\includegraphics[height=32mm]{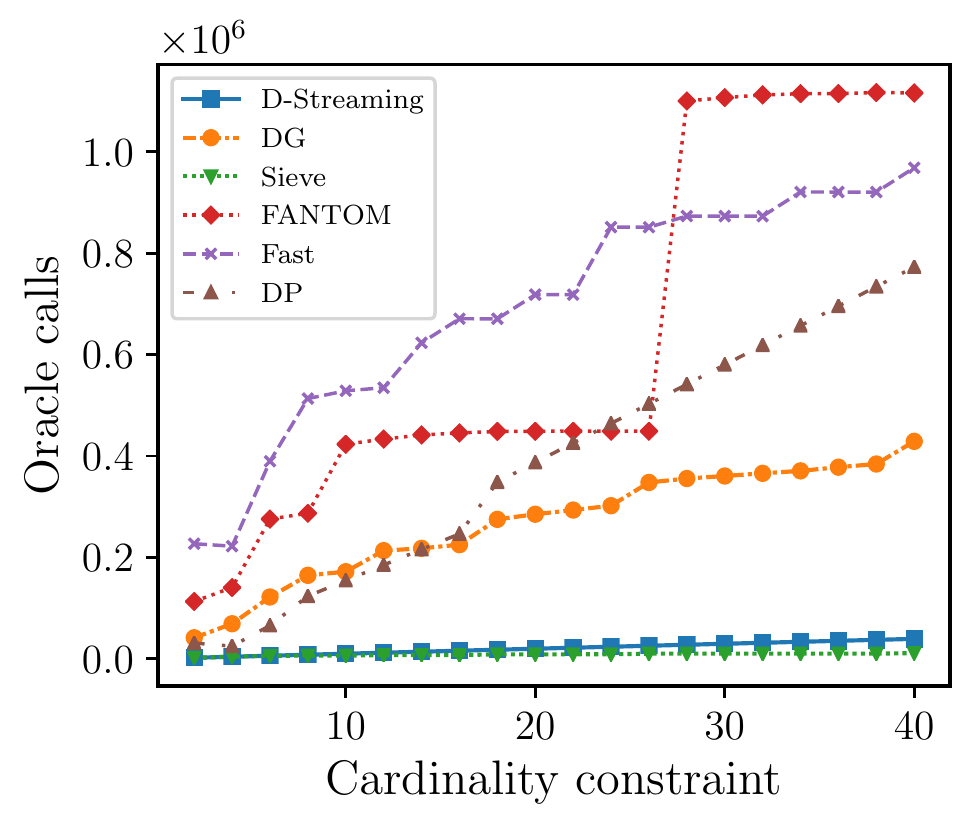}\label{fig:yelp-o-0}}
	\subfloat[$\ell=$  distance to airport] {\includegraphics[height=32mm]{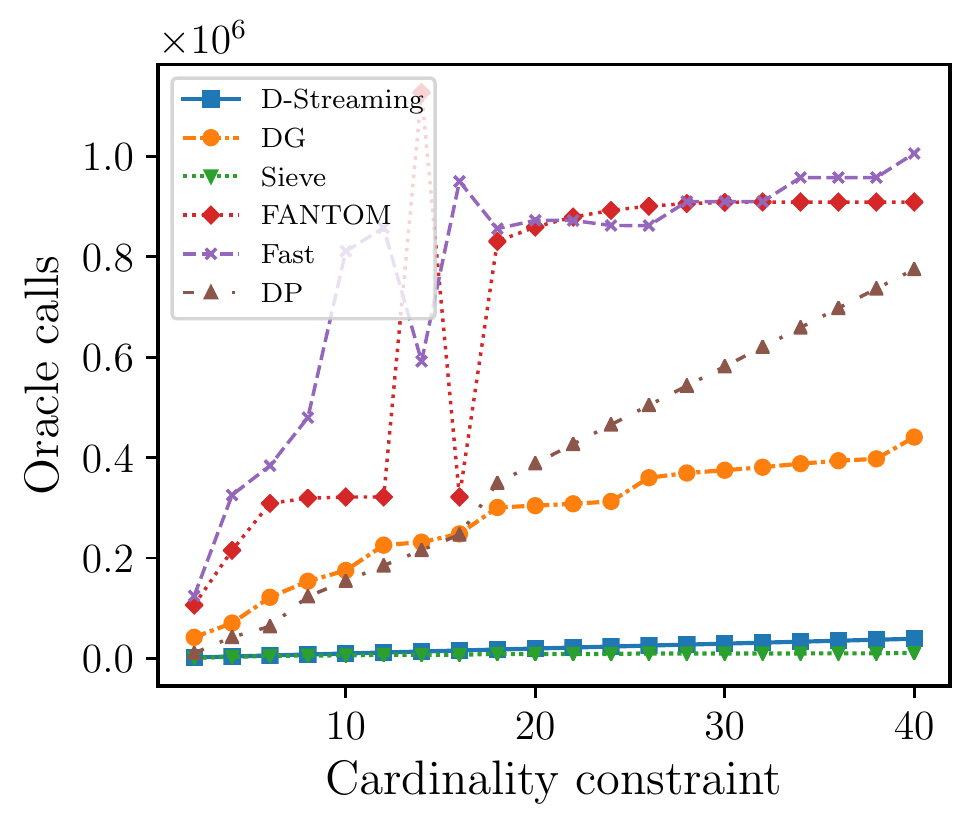}\label{fig:yelp-o-1}}
	\caption{Yelp location summarization: data points are locations from six different cities. For the linear costs we consider two different cases: 1) distances to the downtown in each city, 2) distances to the airport in each city. \label{fig:yelp}}
\end{figure*}

\subsubsection{Movie Recommendation} \label{sec:movilens}

In this application, the goal is to recommend a set of movies to a user, where we know that the user is mainly interested in movies released around 1990. As a matter of fact, we are aware that her all-time favorite movie is Goodfellas (1990).
To design our recommender system, we use ratings from MovieLens users \cite{harper2015movielens}, and apply the method of \citet{lindgren2015sparse} to generate a set of features for each movie.

For a ground set of movies $\cN$, assume $v_i$ represents the feature vector of the $i$-th movie.
Following the same approach we used in \cref{sec:video}, we define a similarity matrix $M$ such that $M_{ij} = e^{- \text{dist}(v_i,v_j)}$, where $\text{dist}(v_i,v_j)$ is the euclidean distance between vectors $v_i, v_j \in \cN$. 
The objective of each algorithm is to select a subset of movies that maximizes 
$f(S) \triangleq g(S) - \ell(S) = \log \det( \mathbf{I} + \alpha M_S ) - \sum_{v \in S} \ell_v$
subject to a cardinality constraint $k$. 
In this application for $\ell(S) = \sum_{v \in S} \ell_v$ we consider two different scenarios: (i) $\ell_v = \lvert 1990 - \textrm{year}_v \rvert$, where $\textrm{year}_v$ denote the release year of movie $v$, and (2) $\ell_v = 10 - \textrm{rating}_v$, where $\textrm{rating}_v$ denotes the IMDb rating of $v$ ($10$ is the maximum possible rating).

From our experimental evaluation in \cref{fig:movie}, we observe that both modeling approaches (directly maximizing the function $f$ and maximizing the function $g$ subject to a knapsack constraint for $\ell$) return solutions with similar objective values.
Besides, we note that the computational complexity of \AlgStream is better than the complexity of the expensive offline algorithms (as it makes only a single pass over the data), but this difference is not very significant for some offline algorithms. 
Nevertheless, \AlgStream always provides better utility than sieve streaming.

\begin{figure*}[ht] 
	\centering  
	\subfloat[$\ell_v = \lvert 1990 - \textrm{year}_v \rvert$] {\includegraphics[height=32mm]{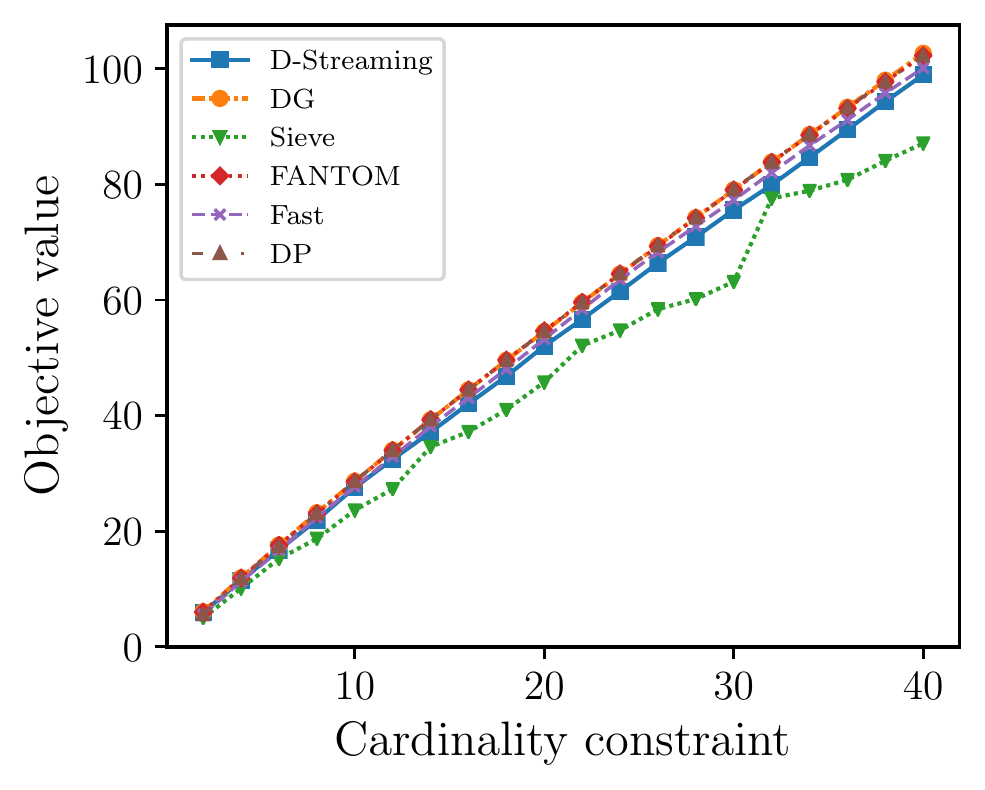}\label{fig:movie-o-2}}
	\subfloat[$\ell_v = \lvert 10 - \textrm{rating}_v \rvert$] {\includegraphics[height=32mm]{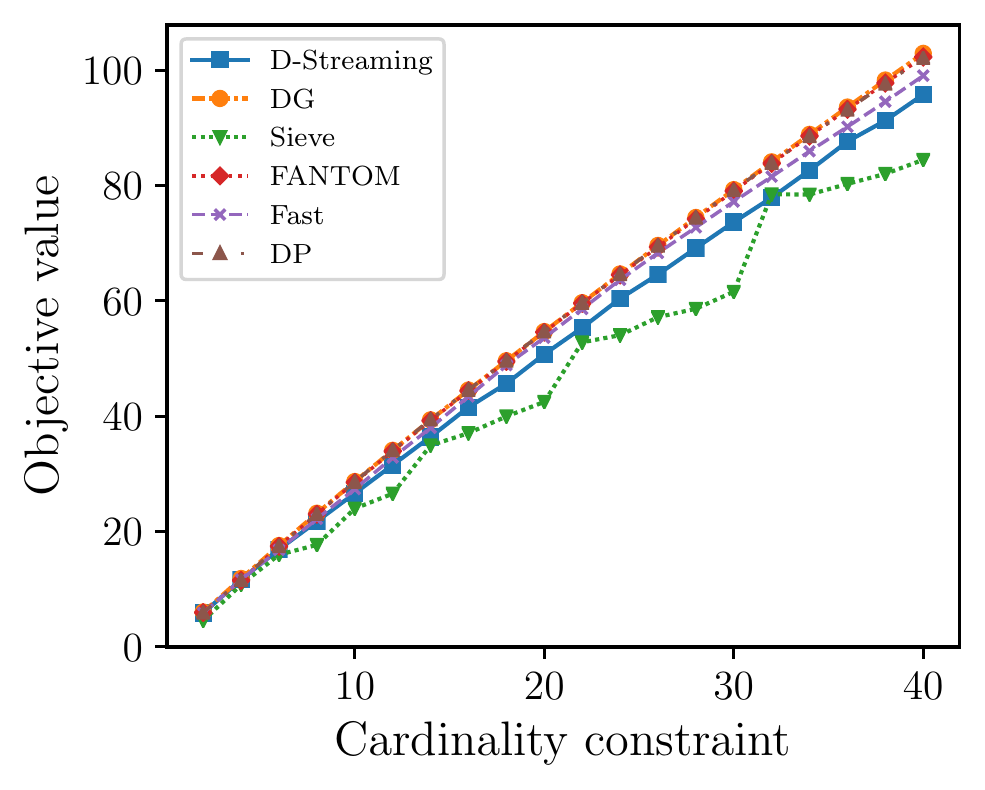}\label{fig:movie-f-1}}
	\subfloat[$\ell_v = \lvert 1990 - \textrm{year}_v \rvert$] {\includegraphics[height=32mm]{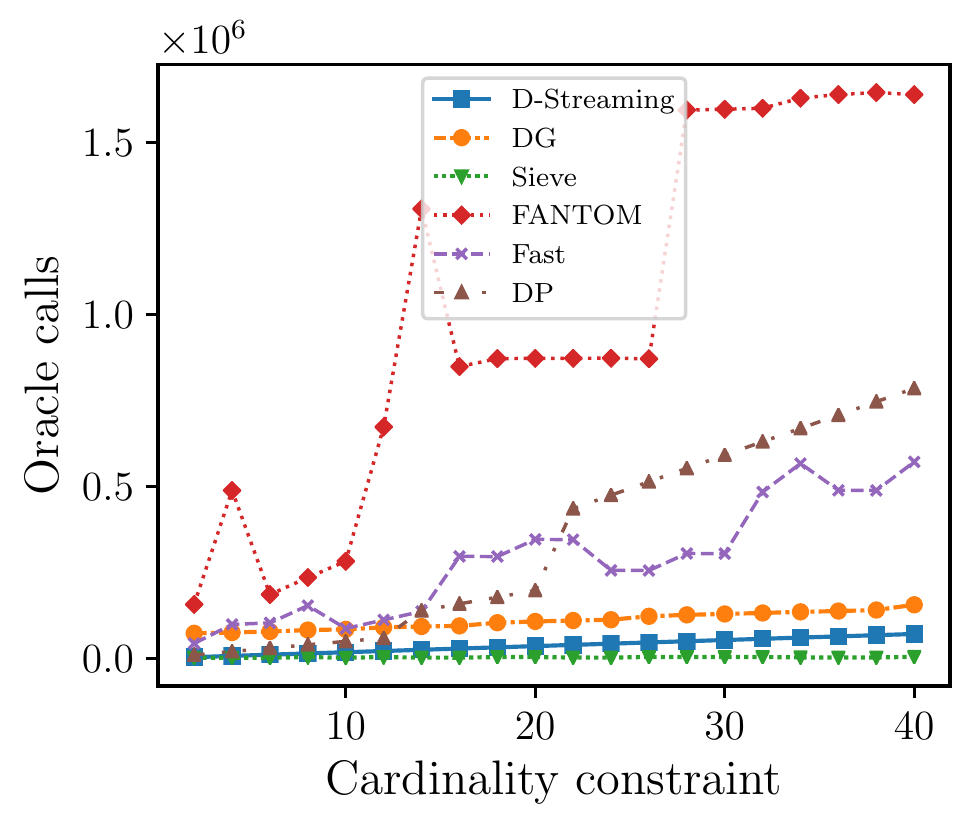}\label{fig:movie-f-2}}
	\subfloat[$\ell_v = \lvert 10 - \textrm{rating}_v \rvert$] {\includegraphics[height=32mm]{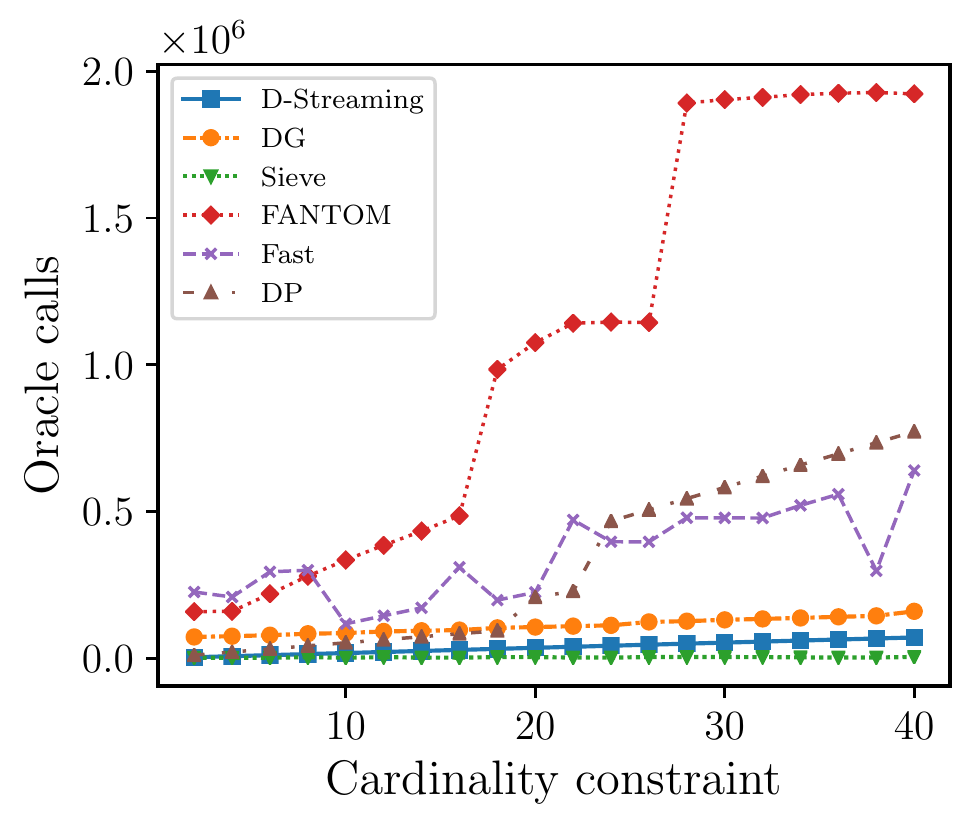}\label{fig:movie-o-1}}
	\caption{Movie recommendation: we compare algorithms for varying cardinality constraint $k$. We use two different linear functions: $\lvert 1990 - \textrm{year}_v \rvert$ and $ 10 - \textrm{rating}_v$, where the goal of first one is to recommend movies with a release year closer to 1990 and the goal of the second linear function is to promote movies with higher ratings.}
	\label{fig:movie}
\end{figure*}

\subsubsection{Twitter Text Summarization} \label{sec:twitter}

There are several news-reporting Twitter accounts with millions of followers.
In this section, our goal is to produce real-time summaries for Twitter feeds of a subset of these accounts.
In the Twitter stream summarization task, one might be interested in a representative and diverse summary of events that happen around a certain date.
For this application, we consider the Twitter dataset provided in \cite{kazemi2019submodular}, where the keywords from each tweet are extracted and weighted proportionally to the number of retweets the post received.
Let $\cW$ denote the set of all existing keywords.
The function $f$ we want to maximize is defined over a ground set $\cN$ of tweets \cite{kazemi2019submodular}.
Assume each tweet $e \in \cN$ consists of a positive value $\textrm{val}_e$ representing the number of retweets it has received (as a measure of the popularity and importance of that tweet) and a set of $l_e$ keywords $\cW_e = \{ w_{e,1}, \cdots, w_{e, l_e}\}$ from $\cW$.
The score of a word $w \in \cW_e $ with respect to a given tweet $e$ is calculated by $\text{score}(w,e) = \text{val}_e$. 
If $w \notin \cW_e $, we assume $\text{score}(w,e) = 0$.
Formally, the function $f$ is defined as follows:
\begin{align*} \label{eq:function-twitter}
f(S) \triangleq g(S) - \ell(S) =  \sum_{w \in \cW} \sqrt{\sum_{e \in S} \text{score}(w,e)} - \sum_{e \in S} \ell_e \enspace,
\end{align*}
where for the linear function $\ell$  two options are considered: (i) $\ell_e = |\printtime- \textrm{T}(e)|$ is the absolute difference (in number months) between time of tweet $e$ and the first of January 2019, (ii) $\ell_e = |\cW_e|$ is the length of each tweet, which enables us to provide shorter summaries.
Note that the monotone and submodular function $g$ is designed to
cover the important events of the day without redundancy (by encouraging diversity in a selected set of tweets) \cite{kazemi2019submodular}.

The main observation from \cref{fig:twitter} is that \AlgStream clearly outperforms the sieve-streaming algorithm and the Greedy Dynamic Program algorithm in terms of objective value, where the gap between their performances grows for larger values of $k$. The utility of other offline algorithms is slightly better than that of our proposed streaming algorithm.
We also see that while distorted greedy is by far the fastest offline algorithm, the computational complexities of both streaming algorithms are negligible with respect to the other offline algorithms.

\begin{figure*}[ht] 
	\centering  
	\subfloat[$\ell_e = |\printtime- \textrm{T}(e)|$] {\includegraphics[height=32mm]{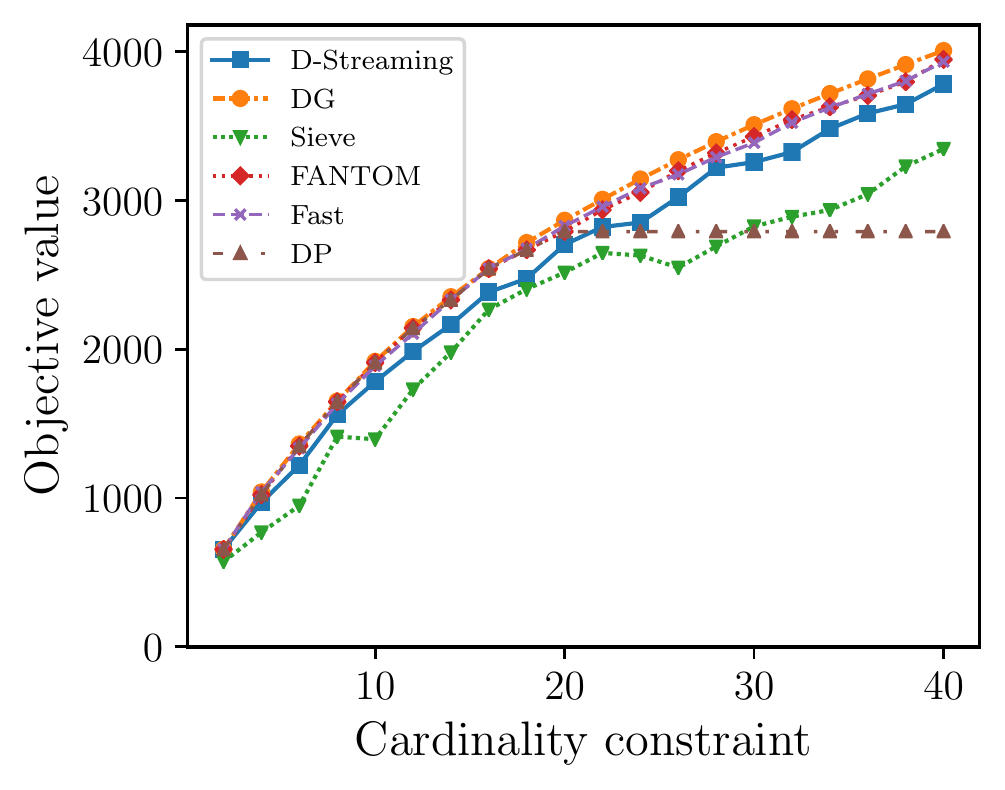}\label{fig:twitter-f-1}}
	\subfloat[$\ell_e = |\cW_e|$] {\includegraphics[height=32mm]{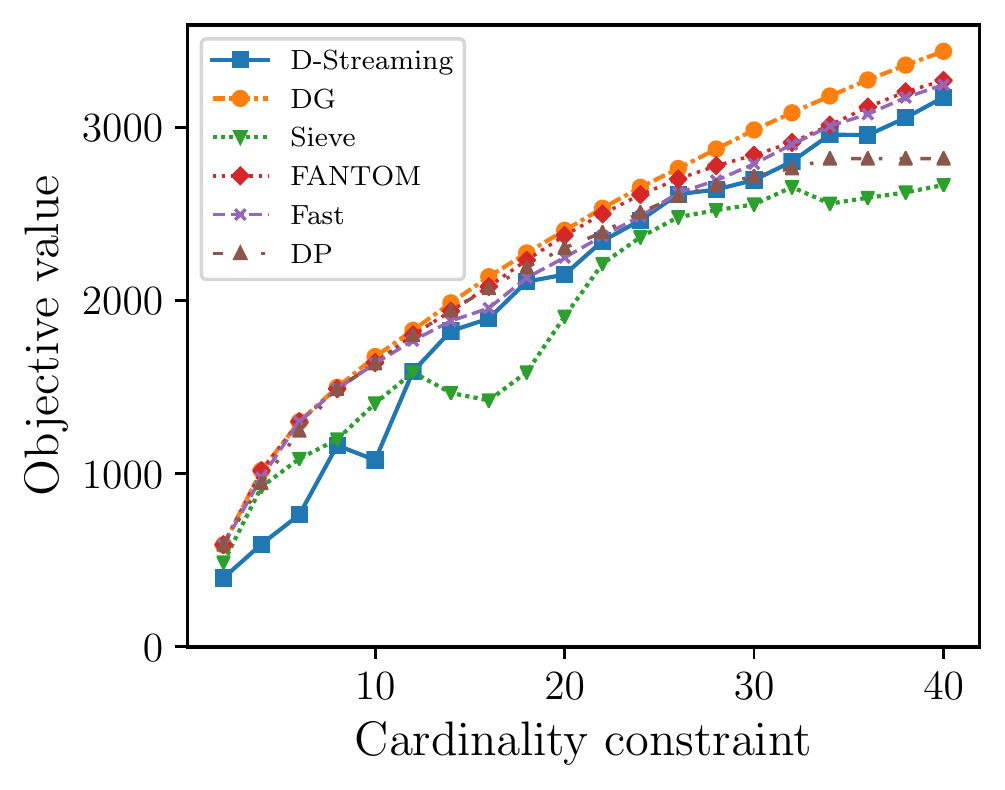}\label{fig:twitter-f-2}}
	\subfloat[$\ell_e = |\printtime- \textrm{T}(e)|$] {\includegraphics[height=32mm]{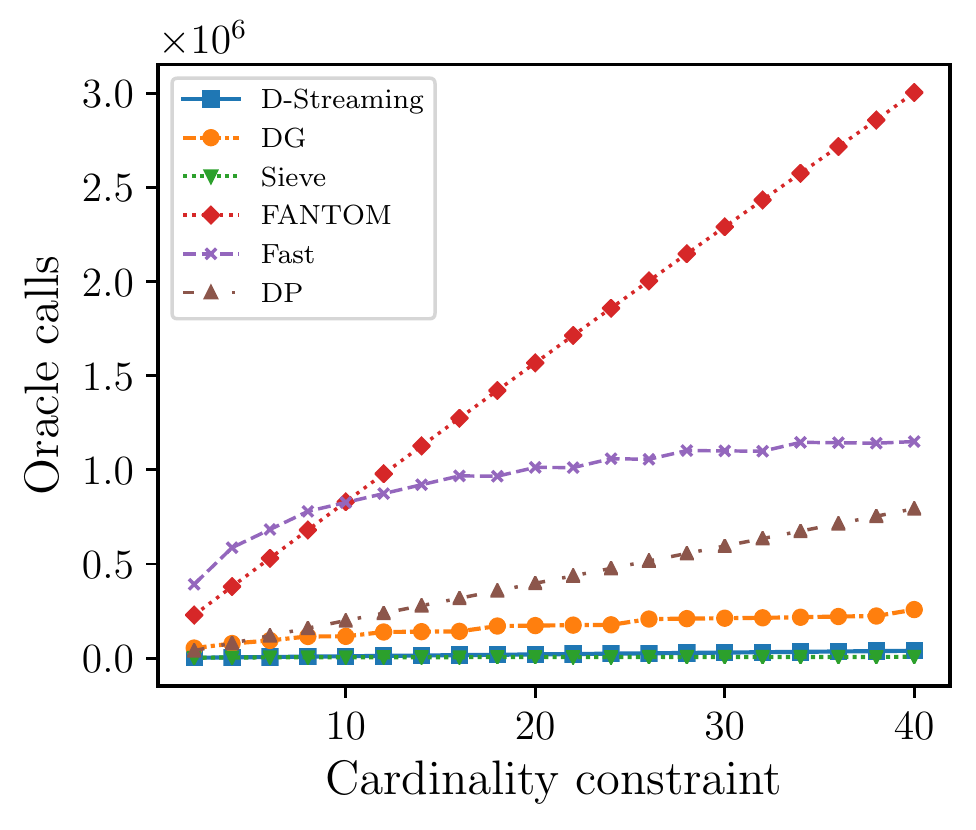}\label{fig:twitter-o-1}}
	\subfloat[$\ell_e = |\cW_e|$] {\includegraphics[height=32mm]{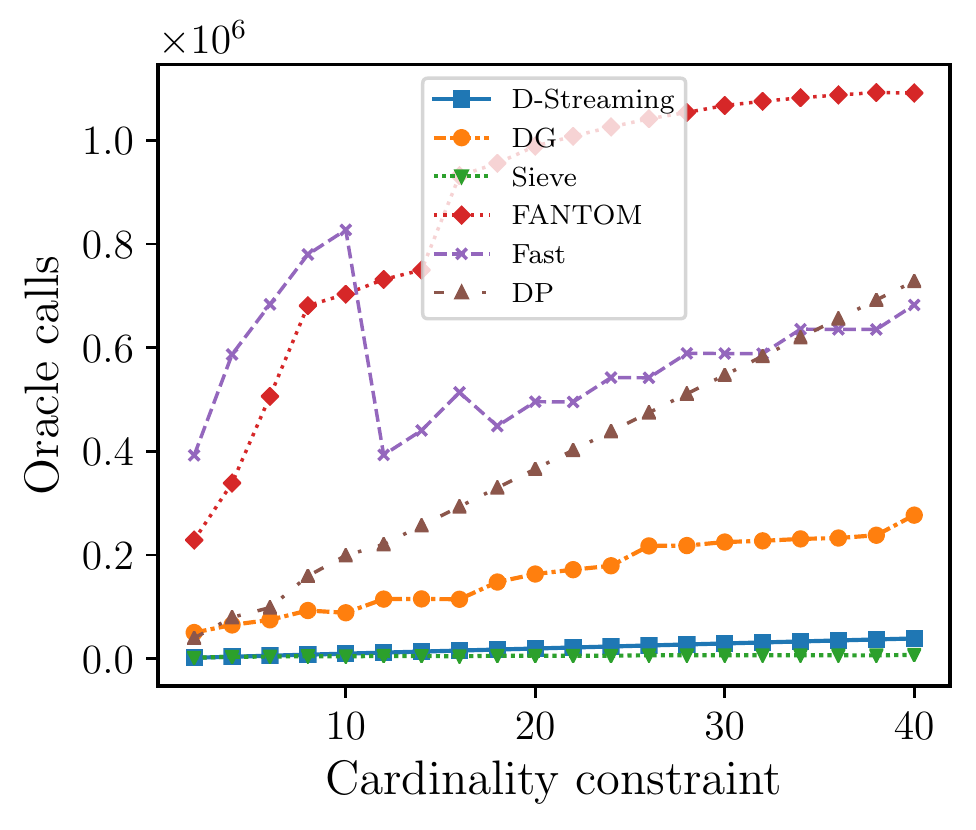}\label{fig:twitter-o-2}}
	\caption{Twitter text summarization: We compare algorithms for varying values of the cardinality constraint $k$. In figures (a) and (c) the linear cost is the difference between the time of the tweet and the first of January 2019. In figures (b) and (d)  the linear cost is the number of keywords in each tweet.}
	\label{fig:twitter}
\end{figure*}

%
\vspace{-5pt}
\section{Conclusion} \label{sec:conclusion}
\vspace{-5pt}
In this paper, we proposed scalable methods   for maximizing a \textit{regularized} submodular function expressed as the difference between a non-negative monotone submodular function $g$ and a modular function $\ell$.
We developed the first one-pass streaming algorithm for maximizing a regularized submodular function subject to a $k$-cardinality constraint with a theoretical performance guarantee, and also presented the first   distributed algorithm that returns a solution $S$ with the guarantee that  $\bE [f(S)] \geq (1 - \eps)  \left[(1 - e^{-1}) \cdot g(OPT) - \ell(OPT) \right]$ in $O(1/\epsilon)$ rounds of MapReduce computation. 
Moreover, even for the unregularized case, our distributed algorithm improves the memory and communication complexity of the existing work by a factor of $O(1/\epsilon)$ while  providing a simpler distributed algorithm and a unifying analysis.  We also empirically studied the performance of our scalable methods on a set of real-life applications, including vertex cover of social networks, finding the mode of strongly log-concave distributions, data summarization, and product recommendation.

	\bibliographystyle{plainnat}
	\bibliography{ScalableSubmodularLinear}
	\appendix
\section{Guessing \texorpdfstring{$\tau$}{tau} in Algorithm~\ref*{alg:threshold-first}} \label{sec:tau}

In this section we explain how one can guess the value $\tau$ in Algorithm~\ref{alg:threshold-first}, which is a value obeying $k\tau \leq h(r) \cdot g(T) - r \cdot \ell(T) \leq (1 + \eps)k\tau$, at the cost of increasing the space complexity of the algorithm by a factor of $O(\eps^{-1} (\log k + \log r^{-1}))$. Like in \cref{sec:thm_proof}, we assume that $h(r) \cdot g(T) - r \cdot \ell(T)$---and thus, also $\tau$---is positive.

Observe that
\begin{align*}
    \max_{u \in \cN} [h(r) \cdot g(\{u\}) - r \cdot \ell(\{u\})]
    \leq{} &
    h(r) \cdot g(T) - r \cdot \ell(T)
    \leq
    \sum_{u \in T} [h(r) \cdot g(\{u\}) - r \cdot \ell(\{u\})]\\
    \leq{} &
    k \cdot \max_{u \in \cN} [h(r) \cdot g(\{u\}) - r \cdot \ell(\{u\})]
    \enspace,
\end{align*}
where the first inequality holds since $\{u\}$ is a candidate to be $T$ for every $u \in \cN$, and the second inequality follows from the submodularity of $g$. Thus, if we knew the value of $\max_{u \in \cN} [h(r) \cdot g(\{u\}) - r \cdot \ell(\{u\})]$ from the very beginning, we could simply run in parallel an independent copy of Algorithm~\ref{alg:threshold-first} for every value of $\tau$ that has the form $(1 + \eps)^i$ for some integer $i$ and falls within the range
\[
	\left[k^{-1} \cdot \max_{u \in \cN} [h(r) \cdot g(\{u\}) - r \cdot \ell(\{u\})], (1 + \eps) \cdot \max_{u \in \cN} [h(r) \cdot g(\{u\}) - r \cdot \ell(\{u\})]\right]
	\enspace.
\]
Clearly, at least one of the values we would have tried obeys $k\tau \leq h(r) \cdot g(T) - r \cdot \ell(T) \leq (1 + \eps)k\tau$, and the number of values we would have needed to try is upper bounded by
\[
	1 + \log_{1 + \eps}\left(\frac{(1 + \eps) \cdot \max_{u \in \cN} [h(r) \cdot g(\{u\}) - r \cdot \ell(\{u\})}{k^{-1} \cdot \max_{u \in \cN} [h(r) \cdot g(\{u\}) - r \cdot \ell(\{u\})]}\right)
	=
	2 + \log_{1 + \eps} k
	=
	O(\eps^{-1} \log k)
	\enspace.
\]


Unfortunately, the value of $\max_{u \in \cN} [h(r) \cdot g(\{u\}) - r \cdot \ell(\{u\})]$ is not known to us in advance. To compensate for this, we make the following two observations. The first observation is that $k^{-1} \cdot \max_{u \in \cN'} [h(r) \cdot g(\{u\}) - r \cdot \ell(\{u\})]$, where $\cN'$ is the set of elements viewed so far, is a lower bound on the value of $k^{-1} \cdot \max_{u \in \cN} [h(r) \cdot g(\{u\}) - r \cdot \ell(\{u\})]$. Following is the second observation, which shows that copies of Algorithm~\ref{alg:threshold-first} with $\tau$ values that are much larger than this lower bound cannot accept any element of $\cN'$, and thus, need not be maintained explicitly.
\begin{observation} \label{obs:ignore}
If $\tau > (\alpha(r) / r) \cdot \max_{u \in \cN'} [h(r) \cdot g(\{u\}) - r \cdot \ell(\{u\})]$, then Algorithm~\ref{alg:threshold-first} accepts no element of $\cN'$.
\end{observation}
\begin{proof}
Algorithm~\ref{alg:threshold-first} accepts an element $u \in \cN'$ if $g(u \mid S) - \alpha(r) \cdot \ell(\{u\}) \geq \tau$. However, the condition of the observation implies
\begin{align*}
	g(u \mid S) - \alpha(r) \cdot \ell(\{u\})
	\leq{} &
	g(\{u\}) - \alpha(r) \cdot \ell(\{u\})
	=
	\frac{\alpha(r)}{r} \cdot [h(r) \cdot g(\{u\}) - r \cdot \ell(\{u\})]\\
	\leq{} &
	\frac{\alpha(r)}{r} \cdot \max_{u \in \cN'} [h(r) \cdot g(\{u\}) - r \cdot \ell(\{u\})]
	<
	\tau
	\enspace,
\end{align*}
where the first inequality follows from the submodularity of $g$, and the equality follows from the following calculation.
\[
	\alpha(r) \cdot h(r)
	=
	\frac{2r + 1 + \sqrt{4r^2 + 1}}{2} \cdot \frac{2r + 1 - \sqrt{4r^2 + 1}}{2}
	=
	\frac{(2r + 1)^2 - (4r^2 + 1)}{4}
	=
	\frac{4r}{4}
	=
	r
	\enspace.
	\qedhere
\]
\end{proof}

The above observations imply that it suffices to explicitly maintain a copy of Algorithm~\ref{alg:threshold-first} for values of $\tau$ that are equal to $(1 + \eps)^i$ for some integer $i$ and fall within the range
\begin{equation} \label{eq:range_alg}
	\left[k^{-1} \cdot \max_{u \in \cN'} [h(r) \cdot g(\{u\}) - r \cdot \ell(\{u\})], \frac{\alpha(r)}{r} \cdot \max_{u \in \cN'} [h(r) \cdot g(\{u\}) - r \cdot \ell(\{u\})]\right]
	\enspace.
\end{equation}
In particular, we know that when the value of $\max_{u \in \cN'} [h(r) \cdot g(\{u\}) - r \cdot \ell(\{u\})]$ increases (due to the arrival of additional elements), we can start a new copy of Algorithm~\ref{alg:threshold-first} for the values of $\tau$ that have the form $(1 + \eps)^i$ for some integer $i$ and now enter the range. By Observation~\ref{obs:ignore}, these instances will behave in exactly the same way as if they had been created at the very beginning of the stream. A formal description of the algorithm we obtain using this method is given as Algorithm~\ref{alg:tau}. We note that the space complexity of this algorithm is larger than the space complexity of Algorithm~\ref{alg:threshold-first} only by an $O(\eps^{-1}(\log k + \log r^{-1}))$ factor because the number of values of the form $(1 + \eps)^i$ that can fall within the range \eqref{eq:range_alg} is at most
\begin{align*}
	\mspace{100mu}&\mspace{-100mu}
	1 + \log_{1 + \eps} \left(\frac{\frac{\alpha(r)}{r} \cdot \max_{u \in \cN'} [h(r) \cdot g(\{u\}) - r \cdot \ell(\{u\})]}{k^{-1} \cdot \max_{u \in \cN'} [h(r) \cdot g(\{u\}) - r \cdot \ell(\{u\})]}\right)
	=
	1 + \log_{1 + \eps} \left(\frac{k \cdot \alpha(r)}{r}\right)\\
	={} &
	1 + \log_{1 + \eps} \left(\frac{k \cdot (2r + 1 + \sqrt{4r^2 + 1})}{2r}\right)
	\leq
	1 + \log_{1 + \eps} (k + k/r)\\
	\leq{} &
	1 + \log_{1 + \eps} k + \log_{1 + \eps}(k/r)
	=
	O(\eps^{-1}(\log k + \log r^{-1}))
	\enspace.
\end{align*}

\begin{algorithm2e}[htb!]
\DontPrintSemicolon
\caption{\AlgStream: Guessing $\tau$} \label{alg:tau}
Let $M \gets -\infty$ and $I \gets \varnothing$. \tcp*{$M$ represents $\max_{u \in \cN'} [h(r) \cdot g(\{u\}) - r \cdot \ell(\{u\})]$ and $I$ is the list of copies of Algorithm~\ref{alg:threshold-first} currently maintained.}
\While{there are more elements in the stream}
{
	Let $u$ be the next element of the stream.\\
	Update $M \gets \max\{M, h(r) \cdot g(\{u\}) - r \cdot \ell(\{u\})\}$.\\
	Let $J = \{i \in \bZ \mid k^{-1}M \leq (1 + \eps)^i \leq r^{-1}M \cdot \alpha(r)\}$.\\
	Delete every copy of Algorithm~\ref{alg:threshold-first} in $I$ corresponding to a value $\tau = (1 + \eps)^i$ for an integer $i$ that now falls outside the set $J$.\\
	Add to $I$ a new copy of Algorithm~\ref{alg:threshold-first} with $\tau = (1 + \eps)^i$ for every integer $i \in J$, unless such a copy already exists there.\\
	Pass the element $u$ to all the copies of Algorithm~\ref{alg:threshold-first} in $I$.
}
\Return{the set $S$ maximizing $g(S) - \ell(S)$ among all the output sets of all the copies of Algorithm~\ref{alg:threshold-first} in $I$}.
\end{algorithm2e}

\end{document}